\documentclass{article}

\usepackage{txie_icml}
\usepackage{natbib}

\usepackage{microtype}
\usepackage{graphicx}
\usepackage{subcaption}
\usepackage{booktabs} 

\usepackage{hyperref}

\usepackage[scaled=.91]{helvet}

\usepackage[accepted]{icml2023}

\usepackage{amsmath}
\usepackage{amssymb}
\usepackage{mathtools}
\usepackage{amsthm}
\usepackage[inline]{enumitem}
\usepackage{adjustbox}

\usepackage[capitalize,noabbrev]{cleveref}

\theoremstyle{plain}
\newtheorem{theorem}{Theorem}[section]
\newtheorem{proposition}[theorem]{Proposition}
\newtheorem{lemma}[theorem]{Lemma}

\theoremstyle{definition}
\newtheorem{definition}[theorem]{Definition}
\newtheorem{assumption}[theorem]{Assumption}
\theoremstyle{remark}
\newtheorem{remark}[theorem]{Remark}
\newcommand{\mypar}[1]{\noindent \textbf{#1} \hspace{1mm}}

\usepackage[textsize=tiny]{todonotes}

\usepackage{xspace}
\def\algo{STR\xspace}

\usepackage[flushleft]{threeparttable}
\usepackage{tablefootnote}
\usepackage{multirow}

\icmltitlerunning{Supported Trust Region Optimization for Offline Reinforcement Learning}

\begin{document}

\twocolumn[
\icmltitle{Supported Trust Region Optimization for Offline Reinforcement Learning}



\icmlsetsymbol{equal}{*}

\begin{icmlauthorlist}
\icmlauthor{Yixiu Mao}{Thu}
\icmlauthor{Hongchang Zhang}{Thu}
\icmlauthor{Chen Chen}{Thu}
\icmlauthor{Yi Xu}{DLUT}
\icmlauthor{Xiangyang Ji}{Thu}
\end{icmlauthorlist}

\icmlaffiliation{Thu}{Department of Automation, Tsinghua University}
\icmlaffiliation{DLUT}{School of Artificial Intelligence, Dalian University of Technology}

\icmlcorrespondingauthor{Yixiu Mao}{myx21@mails.tsinghua.edu.cn}
\icmlcorrespondingauthor{Xiangyang Ji}{xyji@tsinghua.edu.cn}

\icmlkeywords{Machine Learning, ICML}

\vskip 0.3in
]



\printAffiliationsAndNotice{}  

\begin{abstract}
Offline reinforcement learning suffers from the out-of-distribution issue and extrapolation error. Most policy constraint methods regularize the density of the trained policy towards the behavior policy, which is too restrictive in most cases. 
We propose Supported Trust Region optimization~(STR)
which performs trust region policy optimization with the policy constrained within the support of the behavior policy, enjoying the less restrictive \textit{support constraint}.
We show that, when assuming no approximation and sampling error, STR guarantees strict policy improvement until convergence to the optimal support-constrained policy in the dataset. Further with both errors incorporated, STR still
guarantees safe policy improvement 
for each step.
Empirical results validate the theory of STR and demonstrate its state-of-the-art performance on MuJoCo locomotion domains and much more challenging AntMaze domains.
\end{abstract}

\section{Introduction}
Offline Reinforcement Learning~(RL)
aims to learn a policy from a fixed dataset without further interactions.
It can utilize existing large-scale datasets to learn safely and efficiently~\citep{gulcehre2020rl, fu2020d4rl}.
However, this also carries with it a major challenge: 
the evaluation of out-of-distribution (OOD) actions causes extrapolation error~\citep{fujimoto2019off} and overestimation.

Policy constraint methods try to address this by constraining the learned policy to be close to the behavior policy~\citep{wu2019behavior, kumar2019stabilizing,wang2020critic,fujimoto2021minimalist}.
Among them, Weighted Behavior Cloning~(WBC)
performs behavior cloning on the dataset, but assigns different weights to different data points to 
distill a better policy~\citep{chen2020bail,chen2021decision}. 
A common practice is to set the weight as the exponentiated advantage function, leading to Exponentiated Advantage-Weighted Behavior Cloning~(EAWBC)~\citep{wang2018exponentially,peng2019advantage,nair2020awac,wang2020critic,siegel2020keep}.
Mathematically, EAWBC
is equivalent to a policy improvement step in RL with a KL constraint towards 
an implicit baseline policy from which the imitated actions are sampled~(i.e. behavior policy).
As a result, the policy of existing EAWBC methods is implicitly constrained by a \textit{density constraint},
which is typically too restrictive to achieve good performance both theoretically and empirically~\citep{kumar2019stabilizing}.
As a simple example, when the dataset contains a small proportion of the optimal behavior, density constraint will lead to a sub-optimal policy.

On the other hand, from the perspective of optimization process, it is desirable that an RL algorithm can have the \textit{safe policy improvement} guarantee, i.e., have a worst-case performance degradation bound for \textit{each} policy update step.
Benefit from this property, online trust region methods~\citep{schulman2015trust,schulman2017proximal}
have shown supreme performance on both discrete and continuous tasks~\citep{duan2016benchmarking}.
However, in the offline setting with approximation and sampling error, few existing offline RL algorithms can ensure safe policy improvement for each step.

\begin{figure*}[t]
	\centering
	\includegraphics[width=0.98\textwidth]{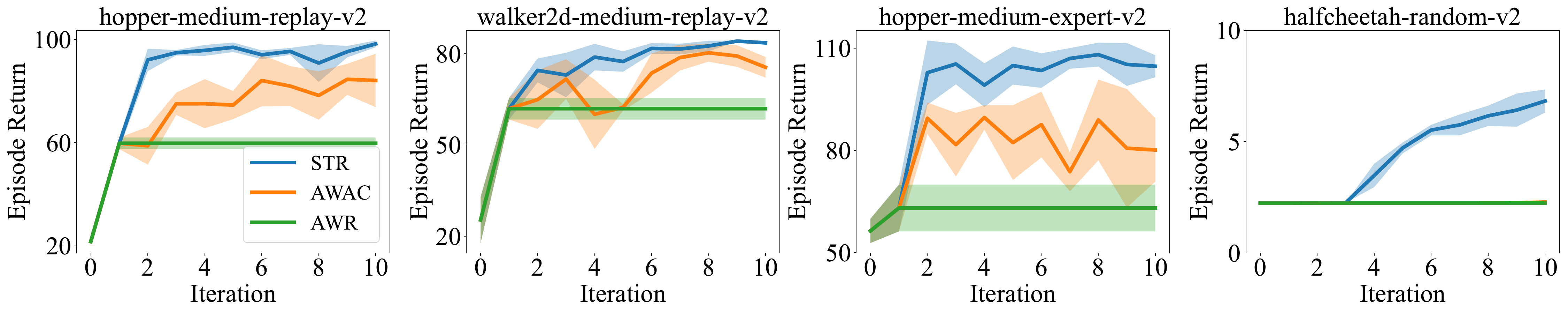}

\caption{\small{Experimental Verification of Theories.
Both policy evaluation and policy improvement
are trained to convergence at each iteration,
and all algorithms adopt the same hyperparameter that controls the constraint strength.
\algo based on supported trust region update
enjoys safe policy improvement for each iteration, and with a less restrictive \textit{support constraint}, the final performance is better.
By contrast, both one-step and multi-step EAWBC methods~(AWR and AWAC)
implicitly satisfying a \textit{density constraint} to the behavior policy have sub-optimal performance.
In halfcheetah-random, AWR and AWAC can hardly make a visible improvement over the behavior policy.
Also, AWAC cannot guarantee safe policy improvement~(most stark in walker2d-medium-replay and hopper-medium-expert).
}}
\label{fig:convergence}
\vspace{-2mm}
\end{figure*}

In this paper, we aim to address the above issues by proposing 
Supported Trust Region optimization~(STR) based on EAWBC.
STR
performs trust region policy optimization with the policy constrained within the support of the behavior policy $\beta$.
This less restrictive \textit{support constraint} allows to seek the best behavior in the dataset and is usually sufficient to mitigate the extrapolation error.
We start from an observation that the analytical policy update form of EAWBC is an equal-support update, which means the updated policy has the same support as the baseline policy.
Based on it, STR utilizes importance sampling on the dataset to mimic sampling from the current policy, which makes the implicit baseline policy in EAWBC become the projection of the current policy on $\beta$.
In this way, by initializing with an estimated behavior policy and adopting a relatively strong policy constraint, the policy of STR is able to deviate from $\beta$ step by step with trust region updates, while still satisfying the support constraint to mitigate extrapolation error. In this process, STR gradually “sharpens” the action distribution of $\beta$, 
giving higher weights to better actions, 
until it converges to the optimal action in the support of $\beta$, while prior EAWBC methods can only get a sub-optimal policy due to the implicit density constraint.

Theoretically, STR enjoys stronger guarantees. Under the same assumptions as prior EAWBC works~(exact tabular $Q$)~\citep{nair2020awac,wang2020critic}, STR guarantees strict policy improvement until convergence to the optimal support-constrained policy, exceeding the prior non-decreasing results
that also have no performance guarantee at convergence.
With $Q$-function approximation and sampling error,
prior EAWBC works lose guarantees, while STR still ensures safe policy improvement for each step.

Empirically, we test \algo on D4RL benchmark~\citep{fu2020d4rl}, including Gym-MuJoCo locomotion domains and much more challenging AntMaze domains. 
STR consistently outperforms state-of-the-art baselines and outperforms prior EAWBC methods by a large margin. 
We also conduct a validation experiment and confirm the theoretical superiority of STR, including the less restrictive support constraint and safe policy improvement~(\cref{fig:convergence}). Besides, compared with prior EAWBC methods, STR maintains good performance over a wide range of the hyperparameter that controls the policy constraint strength.

\section{Preliminaries}
\label{preliminaries}
\mypar{RL.}
In RL, the environment is typically assumed to be a Markov Decision Process~(MDP) $\Mcal=(\mathcal{S}, \mathcal{A}, \Pcal, R, \gamma, d_0)$, with state space $\mathcal{S}$, action space $\mathcal{A}$, transition dynamics $\Pcal: \Scal \times \Acal \to \Delta(\Scal)$, reward function $R: \Scal \times \Acal \to [0, \Rmax]$, discount factor $\gamma \in [0,1)$, and initial state distribution $d_0$~\citep{sutton2018reinforcement}. An agent interacts with the MDP according to a policy $\pi: \Scal \to \Delta(\Acal)$. The goal of the agent is to find a policy that maximizes the expected discounted return: $\eta(\pi)= \mathbb{E}_{\tau\sim\pi}[\sum_{t=0}^{\infty}\gamma^t r_t]$, with $r_t=R(s_t,a_t)$. Here $\tau$ denotes a trajectory $(s_0,a_0,r_0,s_1,\ldots)$ and $\tau\sim\pi$ is shorthand for indicating the distribution of $\tau$ depends on $\pi$: $s_0 \sim d_0, a_t \sim \pi(\cdot|s_t), s_t \sim \Pcal(\cdot|s_t,a_t)$.
For any policy $\pi$, we define the value function as $V^\pi(s)=\mathbb{E}_\pi[\sum_{t=0}^{\infty}\gamma^t r_t | s_0=s]$ and the state-action value function~($Q$-value function) as $Q^\pi(s,a)=\mathbb{E}_\pi[\sum_{t=0}^{\infty}\gamma^t r_t | s_0=s,a_0=a]$. By the boundedness of rewards, we have $0 \leq Q^\pi,V^\pi \leq \tfrac{\Rmax}{1-\gamma} =: \Vmax$. The advantage function is  $A^\pi(s,a)=Q^\pi(s,a)-V^\pi(s)$.
For a policy $\pi$, the Bellman operator $\Tcal^\pi$ is defined as $\left(\mathcal{T}^{\pi} f\right)(s, a):=R(s, a)+\gamma \mathbb{E}_{s' \sim \Pcal(\cdot|s,a)}\left[f\left(s', \pi(s')\right)\right]$, 
where $f(s',\pi(s')):=\E_{a'\sim \pi(\cdot|s')}f(s',a')$.
In addition, we use $d^\pi_t$ to denote the state occupancy of the policy $\pi$ at time step $t$: $d^{\pi}_t(s):=\mathbb{E}_\pi \left[ \mathbb{I}\left[s_{t}=s\right]\right]$, and use $d^\pi$ to denote the normalized and discounted state occupancy: $d^{\pi}(s)=(1-\gamma) \sum_{t=0}^{\infty} \gamma^{t} d^{\pi}_t(s)$. We also define the state-action occupancy with $\rho^\pi_t(s,a)=d^{\pi}_t(s)\pi(a|s)$ and $\rho^\pi(s,a)=d^{\pi}(s)\pi(a|s)$.

\mypar{Offline RL.} In offline RL, the agent 
is provided with a fixed dataset $\Dcal$ 
collected by some behavior policy $\beta$. We define $\Mcal_\Dcal$ as the empirical MDP induced by the dataset $\Dcal$ that uses the empirical transition model based on data counts. Ordinary approximate dynamic programming methods evaluate policy $\pi$ by minimizing temporal difference error~\citep{haarnoja2018soft}, according to the following loss
\begin{equation}
\begin{aligned}
\label{eqn:Q loss}
L_{Q}(\theta) = \mathbb{E}_{(s, a, s')\sim \mathcal{D}}[(&Q_\theta(s,a) - R(s,a)\\
&- \gamma \mathbb{E}_{a'\sim \pi_\phi(\cdot|s')}Q_{\theta'}(s',a'))^2],
\end{aligned}
\end{equation}
where $\pi_\phi$ is a policy parameterized by $\phi$,
$Q_\theta(s,a)$ is a $Q$ function parameterized by $\theta$, and $Q_{\theta'}(s,a)$ is a target network whose parameters are updated via Polyak averaging.

Besides policy evaluation, a typical policy iteration also includes policy improvement.
In continuous action space, a stochastic policy can be updated by reparameterization:
\begin{equation}
\phi \leftarrow \operatorname{argmax}_{\phi} \mathbb{E}_{s \sim \mathcal{D}, \epsilon \sim \mathcal{N}(0,1)}\left[Q_{\theta}\left(s, f_{\phi}\left(\epsilon; s\right)\right)\right]
\end{equation}

In offline RL, OOD actions $a'$ can produce erroneous values for $Q_{\theta'}(s',a')$ 
and lead to an inaccurate estimation of $Q$-values. Then in policy improvement, where the policy is optimized to maximize the estimated $Q_\theta$,  the policy will prefer OOD actions whose values have been overestimated, resulting in poor performance.
To address this issue, WBC methods avoid explicitly maximizing $Q$, but imitate the actions with high $Q$-values from $\Dcal$ to improve $\pi$.

\section{A Unified Framework for EAWBC Works}
\label{sec:unify framework}

\begin{table*}[t]
\caption{Performance improvement guarantees of EAWBC methods in offline RL.}

\label{tab:awbc theory}
\begin{center}
\begin{footnotesize}
\setlength{\tabcolsep}{5pt}
\begin{tabular}{lccccc}
\toprule
\multirow{2}{*}{Method} &\multirow{2}{*}{$\pi_{\text{pe}}$} & \multirow{2}{*}{$\pi_{\text{base}}$}  &\multirow{2}{*}{Step}  &\multicolumn{2}{c}{\begin{tabular}[c]{@{}c@{}}Performance Guarantees\\\end{tabular}}
\\
\cmidrule(r){5-6}
&&&& w/o sampling and $Q$ approx. error  & w/ sampling and $Q$ approx. error          \\ 
\midrule
AWR         &\multirow{2}{*}{$\beta$}  &\multirow{2}{*}{$\beta$} &\multirow{2}{*}{One-step} & \multirow{2}{*}{\begin{tabular}[c]{@{}c@{}}\textit{Strictly} increasing for \textit{one-step}\\ w/o final performance guarantee\end{tabular}} &\multirow{2}{*}{Safe improvement for \textit{one-step}}  \\
MARWIL      &            &            \\ 
\midrule
AWAC        &\multirow{2}{*}{$\pi_i$}  &\multirow{2}{*}{$\beta$} &\multirow{2}{*}{Multi-step} &\multirow{2}{*}{\begin{tabular}[c]{@{}c@{}}\textit{Non-decreasing} for \textit{each step}\\ w/o final performance guarantee\end{tabular}} &\multirow{2}{*}{No guarantee}  \\
CRR         &             &                                \\
\midrule
ABM         &$\pi_i$ &$\propto \beta \exp(\hat{A}^{\pi_i})$ &\multirow{1}{*}{Multi-step} &No guarantee &No guarantee        \\
\midrule
\textbf{\algo~(ours)} &$\pi_i$  &$\operatorname{Proj}_{\operatorname{supp}(\beta)}(\pi_i)$  &Multi-step &\begin{tabular}[c]{@{}c@{}}\textit{Strictly} increasing for \textit{each step}\\ to support-constrained optimal \end{tabular} &Safe improvement for \textit{each step} \\
\bottomrule
\end{tabular}
\end{footnotesize}
\end{center}
\end{table*}

In this section, we propose a unified framework for all prior EAWBC methods and point out their limitations. Later in \cref{sec:STR}, we will present our proposed algorithm based on this framework.
All derivations of this section could be found in \cref{app:EAWBC framework}.

The unified framework follows a policy iteration paradigm.
At the $i^{th}$ iteration, after evaluating some policy $\pi_{\text{pe}}$ to get an advantage estimate $\hat{A}^{\pi_{\text{pe}}}$, it solves the following constrained optimization problem to update the policy, where $\pi_{\text{base}}$ is some baseline policy:
\begin{equation}
\begin{aligned}
\label{eqn:unify}
& \pi_{i+1} = \argmax_\pi \underset{a \sim \pi}{\E} [\hat{A}^{\pi_{\text{pe}}} (s,a)] \\
s.t.~ \mathrm{D_{KL}}&\textstyle(\pi\|\pi_{\text{base}})[s] \leq \epsilon,\quad
\sum_a \pi(a|s)=1,~ \forall s
\end{aligned}   
\end{equation}
The optimization problem above has a closed-form solution:
\begin{equation}
\begin{aligned}
\label{eqn:unify close form}
     \pi_{i+1}(a|s) = &\pi_{\text{base}}(a|s)f(s,a;\pi_{\text{pe}})\\
     \text{where } f(s,a;\pi_{\text{pe}})&:=\textstyle\frac{1}{Z(s)}\exp\left(\frac{\hat{A}^{\pi_{\text{pe}}} (s,a)}{\lambda^*(s)}\right)
\end{aligned}
\end{equation}
Here $Z(s)=\sum_a \pi_{\text{base}}(a|s)\exp\left(\frac{ \hat{A}^{\pi_{\text{pe}}} (s,a)}{\lambda^*(s)}\right)$ is the per-state normalizing factor,
and $\lambda^*(s)$ is the solution of the following convex dual problem:
\begin{equation}
     \min_{\lambda \geq 0}~ \textstyle\epsilon \lambda + \lambda \log\left[\sum_a \pi_{\text{base}}(a|s) \exp\left(\frac{\hat{A}^{\pi_{\text{pe}}} (s,a)}{\lambda}\right)\right]
\end{equation}

In practice, the non-parametric solution in \Eqref{eqn:unify close form} can be projected onto the parametric policy class by minimizing the KL divergence:
\begin{equation}
    \argmin_{\phi} \underset{s \sim \Dcal}{\E} \left[\mathrm{D_{KL}}(\pi_{i+1}(\cdot|s)\|\pi_\phi(\cdot|s))\right]
\label{eqn:unify KL}
\end{equation}
By the relationship between $\pi_{i+1}$ and $\pi_{\text{base}}$ in \Eqref{eqn:unify close form},
\Eqref{eqn:unify KL} is equivalent to maximizing the following unified EAWBC objective:
\begin{equation}
\begin{aligned}
      & J_{U}(\phi)=
       \underset{s \sim \Dcal,a \sim \pi_{\text{base}}}{\E} \left[ f(s,a;\pi_{\text{pe}})\log(\pi_\phi(a|s))\right]
    \label{eqn:eawbc obj unify}
\end{aligned}
\end{equation}
All prior EAWBC methods except ABM~\citep{siegel2020keep} adopt the behavior policy $\beta$ as $\pi_{\text{base}}$ so that they can sample state-action pairs directly from $\Dcal$ to optimize $J_{U}(\phi)$.
Their main difference is what policy is selected as the evaluation policy $\pi_{\text{pe}}$. Typically, AWR~\citep{peng2019advantage} and MARWIL~\citep{wang2018exponentially} use $\beta$ as $\pi_{\text{pe}}$, while AWAC~\citep{nair2020awac} and CRR~\citep{wang2020critic} use the current policy $\pi_i$ as $\pi_{\text{pe}}$. 
Therefore, AWR and MARWIL can be classified as one-step methods that simply perform one step of policy improvement using an on-policy estimate $Q^\beta$~\citep{brandfonbrener2021offline}, while AWAC and CRR belong to multi-step methods that repeatedly evaluate off-policy $Q^\pi$.

However, due to the correspondence between the constrained optimization problem~(\Eqref{eqn:unify}) and the WBC problem~(\Eqref{eqn:eawbc obj unify}), when choosing $\beta$ as $\pi_{\text{base}}$, the maximization of $J_{U}(\phi)$ 
implicitly regularizes $\pi$'s density towards $\beta$.
Density constraint is overly restrictive in many cases~\citep{kumar2019stabilizing}.
With a small or moderate $\epsilon$, due to 
this implicit KL constraint,
the learned policy may be highly sub-optimal. The following lemma confirms this statement.
\begin{lemma}\label{lem:kl constrained eta}
If $\mathrm{D_{KL}}(\pi(\cdot|s)\|\beta(\cdot|s)) \leq \epsilon , \forall s$ is guaranteed, then the performance $\eta$ has the following bound
\begin{equation}
\textstyle
    \eta (\pi) \leq \eta(\beta)+\frac{\Vmax}{\sqrt{2}(1-\gamma)} \sqrt{\epsilon}
\end{equation}
\end{lemma}
To be less pessimistic,
$\epsilon$ needs to be large
and $\lambda^*(s)$ in \Eqref{eqn:unify close form} will be close to $0$.
Then when $\hat{A}^{\pi_{\text{pe}}}$ is estimated poorly, this will lead to a large and catastrophic policy update~\citep{duan2016benchmarking}.
We point out this trade-off between optimality and stability is the main reason that prior methods output a sub-optimal policy in practice. Our experiments in \cref{sec:temp} support this claim.

To relax this implicit density constraint a bit, ABM~\citep{siegel2020keep} first learns a policy $\pi_{{abm}}$ at iteration $i$ by maximizing $J_{U}(\phi)$
with $\pi_{\text{base}} = \beta, \pi_{\text{pe}}=\pi_i$. Then it obtains the next iterate $\pi_{i+1}$ by maximizing $J_{U}(\phi)$ again with $\pi_{\text{base}} = \pi_{{abm}}, \pi_{\text{pe}}=\pi_i$. 
In this way, ABM intuitively relaxes one density constraint to two coupled ones. 
However, it is still a density constraint method in essence, 
and the second optimization needs to imitate the actions from $\pi_{{abm}}$ rather than the dataset, which will bring extrapolation error as we find empirically in \cref{sec:experiments}.

From a theoretical perspective, the guarantees of prior EAWBC works are not strong.
It is proved in~\citet{wang2020critic} that when assuming exact tabular $Q$ estimation,
the multi-step methods~(CRR and AWAC) lead to a non-decreasing $Q$.
However even with this strong assumption,
these works cannot
prove strict monotonicity
and have no performance guarantee at convergence.
On the other hand, 
when considering sampling error and $Q$-function approximation, 
we always get an approximated and inaccurate $Q$, 
which is the main reason why offline RL algorithms suffer from extrapolation error and overestimation. 
However, no prior EAWBC works analyze this important
setting. 
We find that with an approximate $Q$, only one-step EAWBC methods~(AWR, MARWIL) can ensure
safe policy improvement for that only one step, and
those multi-step methods fail to guarantee similar results.
We summarize in \cref{tab:awbc theory} the theories of EAWBC methods including our proposed STR to present a clear comparison.
For a more detailed discussion of related works, please see \cref{app:related works}.

\section{Supported Trust Region Optimization for Offline RL with Safe Policy Improvement} \label{sec:STR}

In this paper, we will 
relax the implicit density constraint of EAWBC methods to a support constraint, and propose a stable Supported Trust Region algorithm - STR. Benefit from the less restrictive support constraint, STR enjoys much stronger guarantees than prior EAWBC methods. In addition, STR imitates the samples totally from the dataset, thus avoiding the extrapolation error issue of ABM.

\begin{definition}[Support-constrained policy]
\label{def:support contraint}
The support-constrained policy class $\Pi$ is defined as 
\begin{equation}
\Pi = \{ \pi ~|~ \pi( a | s) = 0 \text{ whenever } \beta( a | s) = 0 \}
\end{equation}
\end{definition}
The support constraint set $\Pi$ actually includes the KL density constraint set $\Pi_d = \{ \pi ~|~ \mathrm{D_{KL}}(\pi\|\beta)[s] \leq \epsilon \}$ of prior EAWBC works, which means the support constraint is less restrictive.
It allows to seek the best behavior in the dataset and is usually sufficient to mitigate the extrapolation error.
Following prior works~\citep{kumar2019stabilizing}, we also define the optimal support-constrained policy $\pi^*_{\Pi}$.

\begin{definition}[Optimal support-constrained policy]
\label{def:optimal support contraint policy}
The optimal support-constrained policy $\pi^*_{\Pi}$ is defined as:
\begin{equation}
\label{eqn:support constrained optimal pi}
 \pi^*_{\Pi}(\cdot|s):=\argmax _{\pi \in \Pi} Q^*_{\Pi}(s, \pi(s))  
\end{equation}
where $Q^*_{\Pi}$ satisfies the support-constrained Bellman optimality equation:
\begin{equation*}
\label{eqn:support constrained optimal Q}
Q^*_{\Pi}(s,a)=R(s, a)+\gamma \underset{s' \sim P(\cdot|s, a)}{\E}\left[\max _{\pi \in \Pi} Q^*_{\Pi}(s', \pi(s'))\right]
\end{equation*}
\end{definition}
Our key observation is that the closed-form policy update of EAWBC~(\Eqref{eqn:unify close form}) is an equal-support update, which means  $\operatorname{supp}(\pi_{i+1})=\operatorname{supp}(\pi_{\text{base}})$
\footnote{It follows directly as $f(s,a;\pi_{\text{pe}})>0$. 
We use the exact definition of support~$(=0)$.}.
Considering this, if we initialize $\pi_1$ as $\beta$ and choose the current policy $\pi_i$ to be $\pi_{\text{base}}$,
by a recursive argument, $\pi_i$ is still within the support of $\beta$, but the density value of $\pi_i$ can deviate much from $\beta$, even with a small constraint constant $\epsilon$. 
It allows $\pi_i$ to seek the best behavior in the dataset with small update steps.
Therefore, we consider to optimize the following EAWBC objective at iteration $i$, with $\pi_{\text{base}} = \pi_i, \pi_{\text{pe}}=\pi_i$.
\begin{equation}
\begin{aligned}
  J(\phi)=
  \underset{s \sim \Dcal,a \sim \pi_{i}}{\E} \left[ f(s,a;\pi_{i}) \log(\pi_\phi(a|s))\right]
    \label{eqn:eawbc obj pi_i}
\end{aligned}
\end{equation}
However, in practice, 
various errors 
may make $\pi_i$ deviate from $\beta$'s support. Even worse, this deviation will accumulate with iterations, eventually leading to large extrapolation error. To address this issue, 
rather than $\pi_{\text{base}} = \pi_i$, we expect the baseline policy $\pi_{\text{base}}$ to be the projection of $\pi_i$ on $\beta$. We find that Importance Sampling~(IS) can satisfy this requirement exactly.
Instead of sampling from $\pi_i$ to optimize $J(\phi)$, STR adopts IS to sample from $\beta$ and weight the objective by an IS ratio $\pi_i(a|s)/\beta(a|s)$.
\begin{equation}
\begin{aligned}
\textstyle
      J_{IS}(\phi)=
      \underset{s,a \sim \Dcal}{\E} \left[ \frac{\pi_i(a|s)}{\beta(a|s)}f(s,a;\pi_{i})\log(\pi_\phi(a|s))\right]
    \label{eqn:eawbc obj is}
\end{aligned}
\end{equation}
This IS operation is crucial 
to reduce extrapolation error, because it implicitly only weighted-imitates the in-$\beta$-support actions of $\pi_i$.
We give a more detailed explanation here.
IS computes $\mathbb{E}_{q}[p({x})f(x)/q({x})]$ to estimate $\mathbb{E}_p f(x)$.
When $\operatorname{supp}(p) \subseteq \operatorname{supp}(q)$ holds, IS is unbiased. However when the support condition does not hold, IS actually computes $\int_{\operatorname{supp}(q)} p(x)f(x) dx$. 
For the EAWBC objective, $\pi_i$ is $p$ and $\beta$ is $q$. Maximizing the IS objective $J_{IS}(\phi)$ in \Eqref{eqn:eawbc obj is} is equivalent to maximizing
\begin{equation}
    \tilde{J}(\phi)= \underset{s \sim \Dcal,a \sim \tilde{\pi}_{i}}{\E} \left[ f(s,a;\pi_{i})\log(\pi_\phi(a|s))\right]
\end{equation}
where $\tilde{\pi}_i:=\operatorname{Proj}_{\operatorname{supp}(\beta)}(\pi_i)$ is obtained by projecting $\pi_i$ onto $\beta$'s support:
\begin{equation*}
    \tilde{\pi}_i(a|s) = \frac{\mathbb{I}[\beta(a|s)>0]\pi_i(a|s)}{\sum_a \mathbb{I}[\beta(a|s)>0]\pi_i(a|s)}
\end{equation*}
Therefore, if the support constraint is violated at some iteration $i$ due to various errors: $\operatorname{supp}(\pi_{i}) \not\subseteq \operatorname{supp}(\beta)$,  $J_{IS}(\phi)$ will only weighted-imitate the in-$\beta$-support actions of $\pi_i$ and automatically pull $\pi_{i+1}$ back into $\beta$'s support, thus mitigating the extrapolation error in practice.

\begin{algorithm}[t]
\caption{\algo(Tabular)}
\label{alg:tabular}
\begin{algorithmic}
\STATE {\bfseries Input:} Offline dataset $\Dcal$, behavior policy $\beta$, constant $\epsilon$.
\STATE Initialize policy $\pi_1$ with $\beta$.
\FOR{$i = 1,2,\dotsc,N$}
\STATE Policy evaluation: 
\STATE \qquad compute $\hat{Q}^{\pi_i}$ in empirical MDP $\Mcal_\Dcal$
\STATE Compute $\hat{A}^{\pi_i}$:
\STATE \qquad$ \hat{A}^{\pi_i}(s,a) = \hat{Q}^{\pi_i}(s,a) - \E_{a \sim \pi_i}[\hat{Q}^{\pi_i}(s,a)]$.
\STATE Policy improvement: 
\STATE \qquad$\pi_{i+1}(a|s)=\frac{1}{Z(s)}\pi_i(a|s)\exp\left(\frac{\hat{A}^{\pi_i} (s,a)}{\lambda^*(s)}\right)$.
\ENDFOR
\end{algorithmic}
\end{algorithm}

\subsection{Theory of \algo with a tabular $Q$}
\label{sec:str tabular}

\cref{alg:tabular} instantiates a version of \algo in the tabular setting,
which is very concise and can clearly show the theoretical advantages of \algo.

The following proposition characterizes the equal-support property of \algo formally and shows that the $\pi$-induced state~(state-action) distribution also has the same property.

\begin{proposition}\label{prop:tabular pi supp}
For $\pi_i$ in \cref{alg:tabular}, 
$\operatorname{supp}(\pi_{i}) = \operatorname{supp}(\beta), \forall i$. It further implies $\operatorname{supp}(d^{\pi_i}) = \operatorname{supp}(d^{\beta})$ and $\operatorname{supp}(\rho^{\pi_i}) = \operatorname{supp}(\rho^{\beta})$.\footnote{Here $\operatorname{supp}(\rho^{\pi_i}) = \operatorname{supp}(\rho^{\beta})$ is similar to the definition of batch-constrained policies in~\citet{fujimoto2019off}.}
\end{proposition}
All proofs of \cref{sec:STR} could be found in \cref{app:theo_proofs}.

\begin{remark}
\cref{alg:tabular} assumes direct access to $\beta$. Although we can only obtain an estimated $\hat{\beta}=n(s,a)/{n(s)}$ in practice where $n$ is the number of data points in $\Dcal$, it satisfies $\operatorname{supp}(\hat{\beta}) \subseteq  \operatorname{supp}(\beta)$, and all theoretical results of STR will still hold, except that 
the optimal $\beta$-support-constrained policy becomes $\hat{\beta}$-support-constrained one.
Further, when assuming no sampling error~$|\Dcal|=\infty$, $\hat{\beta}$ and $\beta$ are the same.
\end{remark}

Based on \cref{prop:tabular pi supp}, we show that without approximation and sampling error~(tabular $Q$ and infinite $\Dcal$), 
policy evaluation of STR under the empirical MDP $\Mcal_\Dcal$ gives the exact $Q$ function under the true MDP.

\begin{proposition}\label{prop:tabular pe}
In tabular MDP, if the offline dataset $\Dcal$ is generated by a behavior policy $\beta$ and $|\mathcal{D}| = \infty$, then we can have an exact evaluation of $Q^{\pi_{i}}$ for all $\pi_i$ in \cref{alg:tabular}.
\end{proposition}

With an exact tabular $Q$, 
we show that \cref{alg:tabular} guarantees strict policy improvement for each iteration until it converges to the optimal support-constrained policy $\pi^*_{\Pi}$.
With the same assumption as CRR~\citep{wang2020critic}, it is much stronger than the prior non-decreasing results
that also have no performance guarantee at convergence.

\begin{theorem}[Strict policy improvement for each step]
\label{thm:tabular monotone}
If we have the exact tabular estimation of $Q$,
then $\pi_i$ in \cref{alg:tabular} guarantees monotonic improvement:
\begin{equation}
    Q^{\pi_{i+1}}(s,a) \geq Q^{\pi_{i}}(s,a)\quad \forall s,a.
\end{equation}
and the improvement is strict in at least one $(s,a)$ pair until the optimal support-constrained policy $\pi^*_{\Pi}$ is found.
\end{theorem}

\subsection{Theory of \algo with an approximate $Q$}
\label{sec:str approximation}
In this section, we relax the assumption in \cref{sec:str tabular} that prior EAWBC works also make, by incorporating $Q$-function approximation and sampling error.
Specifically, we model the $Q$ function by a value function class $\mathcal{F} \subseteq\left(\mathcal{S} \times \mathcal{A} \rightarrow\left[0, V_{\max }\right]\right)$ and remove the assumption on $|\mathcal{D}|$.
With function approximation, 
we optimize one single objective by weighting each state with $d^{\pi_i}(s)$:
\begin{equation}
\label{eqn:str optimization}
\max_\pi ~ \alpha \underset{\substack{s \sim d^{\pi_i} \\ a \sim \pi}}{\E} [\hat{A}^{\pi_i} (s,a)] -\Vmax \underset{s \sim d^{\pi_i}}{\E} [\mathrm{D_{KL}}(\pi\|\pi_i)]
\end{equation}
For ease of presentation, we use a penalty form rather than a constraint form here. Note that all prior EAWBC works use the penalty form in practice.

The optimization problem above also has a closed form solution, which is irrelevant to $d^{\pi_i}(s)$
and is the same as \Eqref{eqn:unify close form} except for a fixed Lagrange multiplier:
\begin{equation}
\begin{aligned}
\label{eqn:str close form update apporx}
\textstyle
 \pi_{i+1}(a|s)&=\textstyle\frac{1}{Z(s)}\pi_i(a|s)\exp(\frac{\alpha \hat{A}^{\pi_i} (s,a)}{\Vmax})\\
 \normalsize{\text{where }} Z(s)&=\textstyle\sum_a \pi_i(a|s)\exp(\frac{\alpha \hat{A}^{\pi_i} (s,a)}{\Vmax})
\end{aligned}
\end{equation}
\Eqref{eqn:str close form update apporx} is still an equal-support update and the property in \cref{prop:tabular pi supp} still holds\footnote{The proof follows directly as that of \cref{prop:tabular pi supp}.}.
Now we make two standard assumptions in the offline setting~\citep{chen2019information}.
\begin{assumption}[Approximate Completeness]
\label{asm:approx complete}
For any $\pi_i$ in \algo, the following bound holds:
\begin{equation}
    \max_{f\in\Fcal}\min_{g\in\Fcal} \|g - \Tcal^{\pi_i} f\|_{2,\rho^\beta}^2 \leq \epsilon_{\text{complete}}
\end{equation}
\end{assumption}
Here $\|\cdot\|_{2,\rho^\beta}:=\sqrt{\mathbb{E}_{\rho^\beta}\left[(\cdot)^{2}\right]}$ is the $\rho^\beta$-weighted 2-norm
\footnote{We will use the notation $\|f\|_{2,\Dcal}$ for an empirical distribution of the dataset $\Dcal$,where $\|f\|_{2,\Dcal}=\sqrt{\frac{1}{|\Dcal|}\sum_{(s,a,s')\in \Dcal}f(s,a,s')^2}$.}. 
\begin{assumption}[Concentrability]
\label{asm:concentrability}
For the policy $\pi_i$ in \algo, there exists a constant $C$ such that,
\begin{equation}
    \forall t, s, a:  \frac{\rho^{\pi_i}_t(s,a)}{\rho^\beta(s,a)} \leq C
\end{equation}
\end{assumption}
By the equal-support property of STR: $\operatorname{supp}(\rho^{\pi_i}) = \operatorname{supp}(\rho^{\beta})$, this concentrability assumption is very likely to hold under any dataset distribution $\rho^\beta$.

\begin{theorem}[FQE error bound]
\label{thm:fqe}
Under \cref{asm:approx complete} and \cref{asm:concentrability}, with probability at least $1-\delta$, after $K$ iterations of Fitted Q Evaluation~$\rm{(FQE)}$, which initializes $f_0 \in \Fcal$ arbitrarily, and iterates $K$ times:
\begin{equation*}
    \textstyle f_{k} \leftarrow \argmin_{f \in \mathcal{F}} \|f(s, a)-r-\gamma f_{k-1}(s', \pi(s'))\|_{2,\Dcal}
\end{equation*} 
the following bound holds:
\begin{small}
\begin{align}
\label{eqn:fqe}
\|&Q^\pi-f_K\|_{1,\rho^\pi} 
\leq
\frac{1-\gamma^K}{1-\gamma} \sqrt{C\epsilon_{gb}}
+\gamma^{K}\Vmax\\
& \normalsize{\text{where }} \epsilon_{gb} := \frac{{44 \Vmax}^{2} \log (|\mathcal{F}|K / \delta)}{|\Dcal|}+20 \epsilon_{\text {complete}}\nonumber
\end{align}
\end{small}
\end{theorem}
The first term in \Eqref{eqn:fqe} is the sampling and approximation error term, which goes to $0$ 
with more data and a smaller inherent Bellman error
$\epsilon_{\text {complete}}$. The second term is the optimization error term that goes to $0$ with more iterations.

Then we formally present the ``trust region" property of STR.
Although we do not constrain $\mathrm{D_{KL}}(\pi\|\pi_i)$ explicitly, $\pi_{i+1}$ and $\pi_i$ in \algo are close to each other.
\begin{proposition}
[Trust Region]
\label{prop:trust region}
For any $\pi_{i+1},\pi_{i}$ satisfying \Eqref{eqn:str close form update apporx}, the following policy difference bound holds:
\begin{align*}
\mathrm{D_{TV}}(\pi_{i+1} \| \pi_i)&[s] \leq \alpha , \forall s\\
\mathrm{D_{KL}}(\pi_{i} \| \pi_{i+1})&[s] \leq \alpha , \forall s\\
\mathrm{D_{KL}}(\pi_{i+1} \| \pi_{i})[s] \leq ~&\alpha(e^\alpha-e^{-\alpha})/{2} , \forall s
\end{align*}
\end{proposition}

Finally,
we show that with a moderate $\alpha$, \algo is guaranteed to be safe and avoid performance collapse.
\begin{theorem}[Safe policy improvement for each step]
\label{thm:safe policy improvement}
Under \cref{asm:approx complete} and \cref{asm:concentrability}, for $\pi_{i+1},\pi_{i}$ satisfying \Eqref{eqn:str close form update apporx}, with $\epsilon^{\pi_{i+1}}:=\max _{s}|\mathbb{E}_{a \sim \pi_{i+1}}[A^{\pi_i}(s,a)]|$, the  following performance difference bound holds:
\begin{equation*}
\begin{aligned}
&\textstyle \eta(\pi_{i+1})-\eta(\pi_i) \geq \frac{\Vmax}{(1-\gamma)\alpha} \E_{s\sim d^{\pi_i}}[\mathrm{D_{KL}}(\pi_{i+1} \| \pi_i)]\\
&\quad \textstyle - \frac{2\alpha}{1-\gamma} \left(\frac{\gamma\epsilon^{\pi_{i+1}}}{1-\gamma} +\frac{1-\gamma^K}{1-\gamma}  \sqrt{C\epsilon_{gb}} +\gamma^{K}\Vmax\right)
\end{aligned}
\end{equation*}
\end{theorem}

\subsection{The Practical Implementation of \algo}
\label{sec:str implementation}

\begin{algorithm}[t]
\caption{\algo(Practical)}
\label{alg:practical}
\begin{algorithmic}
\STATE {\bfseries Input:} Offline dataset $\Dcal$, constant $\lambda>0$.
\STATE Initialize  behavior policy $\beta_\omega$, policy network $\pi_\phi$, $Q$-network $Q_\theta$, and target $Q$-network $Q_{\theta'}$
\STATE // {\bfseries Behavior Policy Pre-training}
\FOR{each gradient step}
\STATE Update $\omega$ by maximizing $J_{\beta}(\omega)$ in \Eqref{eqn:mle}
\ENDFOR
\STATE // {\bfseries Policy Training}
\STATE Initialize policy $\pi_\phi$ with $\beta_\omega$
\FOR{each gradient step}
\STATE Update $\theta$ by minimizing $L_Q(\theta)$ in \Eqref{eqn:Q loss}
\STATE Update $\phi$ by maximizing $J_{\pi}(\phi)$ in \Eqref{eqn:pi objective}
\STATE Update target network: $\theta' \leftarrow (1-\tau){\theta'} + \tau\theta$
\ENDFOR
\end{algorithmic}
\end{algorithm}

We design the practical algorithm to be as simple as possible to avoid some complex modules confusing our algorithm’s impact on the final performance.

\mypar{Policy Improvement.} Instead of training each iteration to convergence in both the evaluation and improvement stages, practical offline RL algorithms usually take one gradient step.
Therefore, 
we maximize the following objective for policy improvement:
\begin{equation}
\label{eqn:pi objective}
\textstyle
  J_\pi({\phi}) = \underset{(s,a) \sim \Dcal}{\E}  \left[\frac{\bar{\pi}_{\phi}(a|s)}{\beta(a|s)}\exp(\frac{A_\theta (s,a)}{\lambda})\log(\pi_\phi(a|s))\right]
\end{equation}
where $\bar{\pi}_\phi$ means the detach of gradient,
and
\begin{equation}
\textstyle
    A_\theta(s,a):=Q_\theta(s,a)-\E_{\hat{a}\sim \pi_\phi} [Q_\theta(s,\hat{a})]
\label{eqn:compute A}
\end{equation}
We represent the policy with a Gaussian distribution.
In \Eqref{eqn:compute A}, we find that replacing the expectation with 
the policy's mean already obtains good performance. Also, it simplifies the training process without learning a $V$-function. 

Besides, as all prior EAWBC works, we omit the partition function $Z(s)$ in \Eqref{eqn:pi objective},
because it only affects the relative weight of different states in the training objective, not different actions.
It is theoretically unimportant and empirically hard to estimate.
We give a detailed and specific explanation for STR in \cref{app:omit normalize}.
\begin{table*}[t]\centering
\caption{Averaged normalized scores on MuJoCo locomotion and AntMaze tasks over five seeds.
}\label{tab:baselines_mujoco}
\footnotesize
\begin{adjustbox}{max width=460pt}
\begin{tabular}{l||rrrrrrrrr|r}
\hline
Dataset~(v2) & BC    & OneStep & TD3+BC & CQL   & IQL   & AWAC  & CRR  &ABM &MPO  & STR \\ \hline
halfcheetah-med & 42.0 & 50.4 & 48.3 & 47.0 & 47.4 & 47.9 & 47.1 &\textbf{50.9}&39.7&  \textbf{51.8$\pm$0.3} \\
hopper-med & 56.2 &  {87.5} & 59.3 & 53.0 & 66.2 & 59.8 & 38.1&39.4 &0.7& \textbf{101.3$\pm$0.4} \\
walker2d-med & 71.0 &  \textbf{84.8} & 83.7 & 73.3 & 78.3 & 83.1 & 59.7 &17.2&-0.2& \textbf{85.9$\pm$1.1} \\
halfcheetah-med-replay & 36.4 & 42.7 & 44.6 & 45.5 & 44.2 & 44.8 & 44.4&43.4 &\textbf{51.6}&  \textbf{47.5$\pm$0.2} \\
hopper-med-replay & 21.8 &  {98.5} & 60.9 & 88.7 & 94.7 & 69.8 & 25.5 &74.7&48.2&  \textbf{100.0$\pm$1.2} \\
walker2d-med-replay & 24.9 & 61.7 & 81.8 & 81.8 & 73.8 & 78.1 & 27.0 &\textbf{86.9}&3.8&  \textbf{85.7$\pm$2.2} \\
halfcheetah-med-exp & 59.6 & 75.1 & 90.7 & 75.6 & 86.7 & 64.9 & 85.2&70.2 &13.2&  \textbf{94.9$\pm$1.6} \\
hopper-med-exp & 51.7 &  \textbf{108.6} & 98.0 &  {105.6} & 91.5 & 100.1 & 53.0 &1.4& 0.7& \textbf{111.9$\pm$0.6} \\
walker2d-med-exp & 101.2 &  \textbf{111.3} &  \textbf{110.1} &  {107.9} &  \textbf{109.6} & \textbf{110.0} & 91.3 &0.1& -0.3& \textbf{110.2$\pm$0.1} \\
halfcheetah-exp &  {92.9} & 88.2 &  \textbf{96.7} &  \textbf{96.3} &  \textbf{95.0} & 81.7 & 93.5 &17.6&-3.8&  \textbf{95.2$\pm$0.3} \\
hopper-exp &  {110.9} & 106.9 & 107.8 & 96.5 &  {109.4} & 109.5 & 108.7&2.4 &0.7&  \textbf{111.2$\pm$0.3} \\
walker2d-exp & 107.7 &  \textbf{110.7} &  \textbf{110.2} & 108.5 &  \textbf{109.9} & \textbf{110.1} & 108.9 &64.9&-0.3&  \textbf{110.1$\pm$0.1} \\
halfcheetah-rand & 2.6 & 2.3 & 11.0 & 17.5 & 13.1 & 6.1 & 13.6&2.3 &\textbf{27.4}&  {20.6$\pm$1.1} \\
hopper-rand & 4.1 & 5.6 & 8.5 & 7.9 & 7.9 & 9.2 & 16.1 &15.2&\textbf{31.7}&  \textbf{31.3$\pm$0.3} \\
walker2d-rand & 1.2 & 6.9 & 1.6 & 5.1 & 5.4 & 0.2 & 4.9 &2.6&1.6& 4.7$\pm$3.8 \\
\hline
locomotion total & 784.2 & 1041.2 & 1013.2 & 1010.2 & 1033.1 & 975.6 & 817&489.1   &214.6&  \textbf{1162.2} \\ 
\hline \hline
antmaze-umaze & 66.8 & 54.0 & 73.0 & 82.6 & {89.6} & 80.0 & 43.8&87.0 &0.0& \textbf{93.6$\pm$4.0} \\
antmaze-umaze-diverse & 56.8 & 57.8 & 47.0 & 10.2 & 65.6 & 52.0 & 42.8&25.4  &0.0&   \textbf{77.4$\pm$7.2} \\
antmaze-med-play & 0.0 & 0.0 & 0.0 & 59.0 &   {76.4} & 0.0 & 0.4&0.0  &0.0&   \textbf{82.6$\pm$5.4} \\
antmaze-med-diverse & 0.0 & 0.6 & 0.2 & 46.6 &   {72.8} & 0.2 & 0.5&0.2  &0.0&   \textbf{87.0$\pm$4.2} \\
antmaze-large-play & 0.0 & 0.0 & 0.0 & 16.4 & \textbf{42.0} & 0.0 & 0.0&0.0  &0.0&   \textbf{42.8$\pm$8.7} \\
antmaze-large-diverse & 0.0 & 0.2 & 0.0 & 3.2 &   \textbf{46.0} & 0.0 & 0.0&0.0  &0.0&   \textbf{46.8$\pm$7.6} \\\hline
antmaze total & 123.6 & 112.6 & 120.2 & 218   & 392.4 & 132.2 & 87.6   &112.6& 0.0&   \textbf{430.2} \\\hline
\end{tabular}%
\end{adjustbox}
\end{table*}

\mypar{Density Estimator.}
We learn a Gaussian density estimator $\beta_\omega$ for the behavior policy by maximizing
\begin{equation}
\label{eqn:mle}
J_\beta(\omega)=\mathbb{E}_{s,a\sim \mathcal{D}}\log \beta_{\omega}(a|s),
\end{equation}
where $\omega$ is the parameter of the estimated behavior policy. 

\mypar{Policy Initialization.}
Theoretically, \algo needs to initialize the actor with the behavior policy. 
In our implementation, we set the network structure of the actor and $\beta_\omega$ to be the same, and initialize the actor with $\beta_\omega$ directly. 

\mypar{Importance Sampling.}
To reduce the high variance of importance sampling~\citep{precup2001off}, \algo adopts Self-Normalized Importance Sampling~(SNIS) in practice,
which normalizes the IS ratio across the batch.

\mypar{Overall Algorithm.}
Putting everything together, we present a practical version of \algo in \cref{alg:practical}.

\section{Experiments}
\label{sec:experiments}
We test the effectiveness of STR~(\cref{alg:practical}) in terms of performance, safe policy improvement, and hyperparameter robustness using the D4RL benchmark~\citep{fu2020d4rl}. More experimental details are provided in \cref{app:implementation}.

\subsection{Comparisons on D4RL Benchmarks}
We evaluate STR on D4RL in comparison to prior methods. 

\mypar{Tasks.}
We conduct experiments in Gym-MuJoCo locomotion domains and more challenging AntMaze domains.
The latter consist of sparse-reward tasks and require “stitching” fragments of suboptimal trajectories traveling undirectedly to find a path from the start to the goal of the maze.

\mypar{Baselines.}
Our offline RL baselines include Behavior Cloning~(BC), OneStep RL~\citep{brandfonbrener2021offline}, TD3+BC~\citep{fujimoto2021minimalist}, CQL~\citep{kumar2020conservative}, IQL~\citep{kostrikov2022offline}, AWAC~\citep{nair2020awac}, CRR~\citep{wang2020critic}, ABM~\citep{siegel2020keep}, and MPO~\citep{abdolmaleki2018maximum}. 

\mypar{Comparison with Baselines.}
Results
are shown in \cref{tab:baselines_mujoco}. For learning curves, please refer to \cref{app:learning curves}. We find that \algo substantially outperforms state-of-the-art methods. It is worth noting that \algo outperforms prior EAWBC methods by a large margin, which further demonstrates the advantages of the supported trust region update proposed by \algo. Moreover, since ABM and MPO imitate the actions sampled from a learned policy rather than the dataset, they suffer from extrapolation error and fail in most tasks that have a narrow data coverage.

\mypar{Runtime.} We test the runtime of \algo on halfcheetah-medium-replay on a GeForce RTX 3090. The results of \algo and other baselines are shown in \Cref{fig:ablation_runtime}~(Right). 
The runtime of STR
is comparable to other baselines. Note that it only takes two minutes for the pre-training part.

\subsection{Experimental Verification of the Theories}
\label{sec:Experimental Verification of the Theories}

\begin{figure*}[t]
	\centering
	\includegraphics[width=0.99\textwidth]{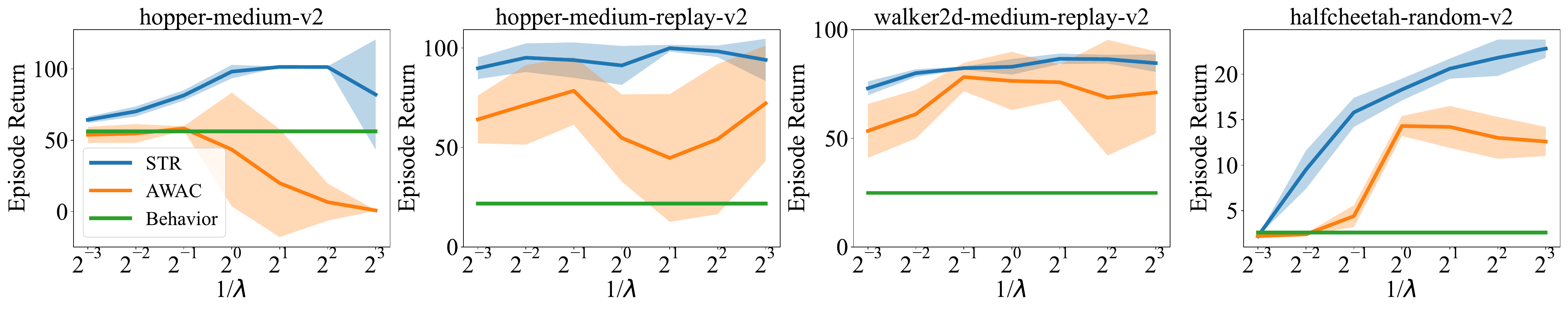}

\caption{\small{Performance of STR~(support constraint) and AWAC~(density constraint) with different constraint strength.
As the abscissa $1/\lambda$ increases, the constraint becomes looser. 
STR is more robust to $\lambda$.
The plots show the average and standard deviation over 5 seeds.
}}
\label{fig:temp}
\vspace{1mm}
\end{figure*}

\begin{figure*}[t]
	\centering
	\includegraphics[width=0.98\textwidth]{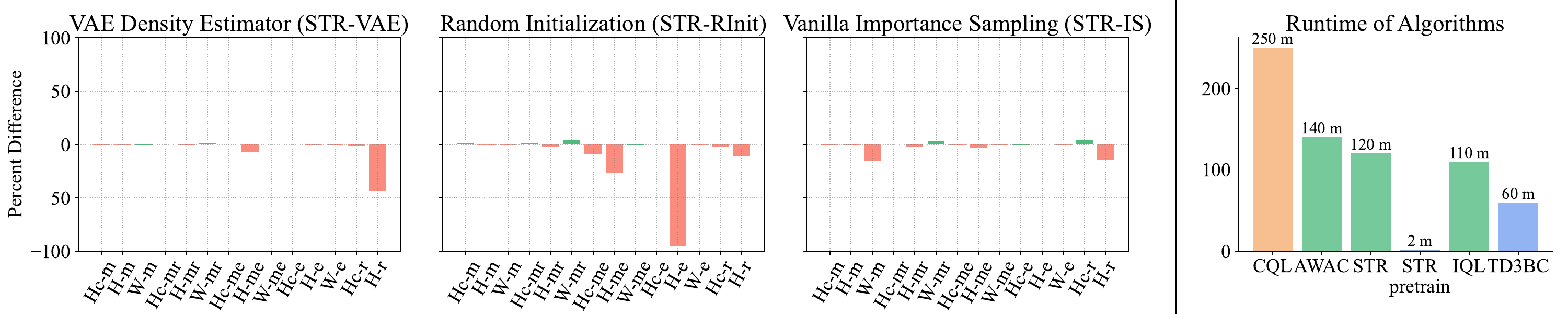}

\caption{\small{(Left)~Percent difference of the performance of an ablation of STR, compared with the original algorithm. Hc = HalfCheetah, H = Hopper, W = Walker2d, m = medium, mr = medium-replay, me = medium-expert, e=expert, r=random.
(Right)~Runtime of various offline RL algorithms for halfcheetah-medium-replay on a GeForce RTX 3090.}}
\label{fig:ablation_runtime}
\vspace{-2mm}
\end{figure*}

In this section, we demonstrate the safe policy improvement property of STR from experiments.
For this, each policy iteration is trained to convergence in both evaluation and improvement parts.
Also, we adopt a relatively large temperature $\lambda$ in \Eqref{eqn:pi objective} that corresponds to a strong constraint.

The results are shown in \cref{fig:convergence}, which well support the theory.
The performance of AWR and AWAC is limited by the implicit density constraint toward the sub-optimal behavior policy~(\cref{lem:kl constrained eta}). 
By contrast, with a less restrictive support constraint, \algo is able to deviate more from the behavior policy to achieve better performance.
In halfcheetah-random, due to the highly sub-optimal behavior policy, AWAC and AWR can hardly make a visible improvement, while STR 
obtains better performance.
Furthermore,
\algo has an approximately monotonic performance improvement process
in all domains,
which demonstrates the safe policy improvement property of STR~(\cref{thm:safe policy improvement}). 
In contrast, AWAC has severe performance drops at some iterations (most stark in walker2d-medium-replay and hopper-medium-expert).

\subsection{Empirical Study on the Constraint Strength}
\label{sec:temp}

We investigate the effect of constraint strength 
on STR and AWAC~(best performance among prior EAWBC methods in \cref{tab:baselines_mujoco}).
For this, we vary the temperature $\lambda$ in \Eqref{eqn:pi objective} which is positively correlated with the constraint strength.
The results are presented in \cref{fig:temp}. Note as the abscissa $1/\lambda$ increases, the constraint becomes looser.

As expected, when the constraint is relatively strong, due to the superiority of the support constraint, the performance of STR is better than that of AWAC, and both algorithms are relatively stable~(with small standard deviations). As the constraint gets slightly weaker, the performance of both algorithms improves. 
However, when the constraint becomes too loose, in hopper-medium both STR and AWAC suffer from overestimation and are unstable~(with large standard deviations), while in hopper-medium-replay and walker2d-medium-replay, STR is more stable than AWAC.
In general, STR not only maintains good performance over a wide range of $\lambda$, but is also better than AWAC for every $\lambda$.

\subsection{Ablation Study}
\label{sec:ablation}

We perform an ablation study over the components in our method, including the behavior density estimator, policy initialization, and importance sampling techniques. The results are shown in \cref{fig:ablation_runtime}.

\mypar{Behavior Density Estimator.}
Following previous works~\citep{fujimoto2019off,wu2022supported}, we consider to replace the Gaussian density estimator $\beta_\omega$ with conditional variational auto-encoder~\citep{sohn2015learning}. We refer to this variant as STR-VAE. 
Overall, the performance of STR-VAE and STR is almost the same, except that STR-VAE obtains worse results on hopper-random.

\mypar{Policy Initialization.}
Consistent with the theory of STR, we initialize the policy with the pre-trained $\beta_\omega$. 
Here we evaluate an STR variant with random policy initialization, termed STR-RInit.
STR-RInit achieves similar performance in most tasks except hopper-expert and hopper-medium-expert. 
The reason is that,
in practice,
the IS operation in STR will implicitly pull $\pi_\phi$ back into $\beta$'s support, no matter how $\pi_\phi$ is initialized. Therefore, an inaccurate $\beta_\omega$ initialization or even random initialization could also achieve good results in most cases.
However in hopper-expert, the data coverage is narrow and overestimation is more likely to happen. We find the random initialization will lead to severe overestimation of $Q$ in hopper-expert.

\mypar{Importance Sampling.}
\algo adopts SNIS to reduce the variance of IS. 
Here we test an STR variant with vanilla importance sampling, termed STR-IS. 
As shown in \cref{fig:ablation_runtime},
the performance of STR-IS is slightly worse than STR due to higher variance, but the difference is not large.

\section{Conclusion}
We propose STR, an offline RL algorithm with supported trust region policy optimization based on EAWBC. 
STR relaxes the density constraint of prior EAWBC works to a support constraint and enjoys stronger theoretical guarantees, including 
strict policy improvement until convergence to the optimal support-constrained policy with an exact $Q$ and safe policy improvement with an approximate $Q$.
Empirical evaluations confirm the theoretical results and demonstrate STR’s SoTA performance on offline RL benchmarks.

\section*{Acknowledgment}
This work was supported by the National Key R\&D Program of China under Grant 2018AAA0102801, National Natural Science Foundation of China under Grant 61827804.

\bibliography{example_paper}
\bibliographystyle{icml2023}
\newpage
\appendix
\onecolumn

\section{Related Works}
\label{app:related works}

\mypar{Offline RL.} In offline RL, a fixed dataset is provided and no further interactions are allowed.
As a result, ordinary off-policy RL algorithms suffer from extrapolation error due to OOD actions~\citep{fujimoto2019off} and have poor performance.
Among various solutions, 
value penalty methods attempt to penalize the $Q$-values of OOD actions~\citep{kumar2020conservative, an2021uncertainty,lyu2022mildly}, while policy constraint methods force the trained policy to be close to the behavior policy by KL divergence~\citep{wu2019behavior}, or by direct behavior cloning~\citep{fujimoto2021minimalist}.
Recently, instead of density policy constraint which is too restrictive in many cases, some works consider the less restrictive support constraint to keep the learned policy within the support of behavior policy by Maximum Mean Discrepancy~(MMD)~\citep{kumar2019stabilizing} or explicit density estimation~\citep{wu2022supported}, 
but their performance still leaves considerable room for improvement.
Another branch of algorithms chooses to perform in-sample learning, by formulating the Bellman target without querying the values of actions not contained in the dataset.
Among them, OneStep RL~\citep{brandfonbrener2021offline} evaluates the behavior policy by SARSA and only performs one-step of constrained policy improvement without off-policy evaluation. 
IQL~\citep{kostrikov2022offline} modifies the SARSA update by expectile regression to approximate an upper expectile of the value distribution, and enables multi-step dynamic programming. 

\mypar{Weighted Behavior Cloning.} Among policy constraint methods,
Weighted Behavior Cloning~(WBC) reduces the RL problem to a supervised learning problem~\citep{emmons2021rvs}.
They modify an imitation learning algorithm either by filtering or weighting the actions in dataset to distill a better policy. 
For the filtering version, BAIL~\citep{chen2020bail} proposes upper-envelope to select good state-action pairs for later imitation learning.
$10\%$BC only uses the top $10\%$ of transitions ordered by episode return to perform behavior cloning~\citep{chen2021decision}.
However, they are intuitive methods and do not have theoretical guarantees of policy improvement.
For the weighted version, prior works typically adopt Exponentiated  Advantage-Weighted Behavior Cloning~(EAWBC), with the weights being determined by the exponentiated advantage estimates. 
Examples of EAWBC methods include MARWIL~\citep{wang2018exponentially}, AWR~\citep{peng2019advantage}, AWAC~\citep{nair2020awac}, CRR~\citep{wang2020critic} and ABM~\citep{siegel2020keep}.
We have summarized their essential similarities and differences in \cref{sec:unify framework}. At the implementation level, MARWIL and AWR are similar one-step methods and only differ in advantage estimation: with a learned value function $V^{\beta}$, MARWIL uses the single-path advantage estimate while AWR uses TD($\lambda$) to approximate the episode return.
CRR and AWAC are concurrent multi-step methods and the difference is that CRR uses distributional $Q$-function~\citep{barth-maron2018distributional}.

To some extent, our algorithm STR shares some similarities with an online algorithm MPO~\citep{abdolmaleki2018maximum}, 
which starts from an inference perspective and essentially 
weighted-clones the actions from the current policy for policy improvement.
In the offline setting, MPO does not satisfy the support constraint and suffers from large extrapolation error, 
as we find empirically in \cref{sec:experiments}. Also, MPO has no theoretical analysis for the offline setting and its implementation is very complicated. 
On the other hand, STR initializes the policy with the behavior policy and 
utilizes importance sampling on the dataset to mimic actions from the current policy.
It enforces the policy of STR within the support of the behavior policy, thus mitigating the extrapolation error and offering excellent empirical performance.

\mypar{Trust Region and Safe Policy Improvement.}
In online RL, trust region methods, with two typical embodiments of Trust Region Policy Optimisation~(TRPO)~\citep{schulman2015trust} and Proximal Policy
Optimisation~(PPO)~\citep{schulman2017proximal}, have shown impressive performance on
both discrete and continuous tasks~\citep{duan2016benchmarking}. 
They optimize the policy within a trusted neighborhood of the current policy, which empirically avoids taking aggressive updates towards risky directions, and theoretically guarantees safe policy improvement at each step~\citep{schulman2015trust,achiam2017constrained}. However in offline RL, 
there are few studies about trust region optimization.
Besides, with approximation and sampling error, few offline RL algorithms can ensure safe policy improvement for each step.
Recently, there are several offline RL works~\citep{liu2020provably,xie2021bellman,cheng2022adversarially} that guarantee safe policy improvement over the behavior policy.
In contrast, we consider safe policy improvement with respect to the current policy, at each policy update step.

\section{Proofs}
\label{app:theo_proofs}
We start out with the performance difference lemma~\citep{Kakade2002approximately}
that shows that the difference in policy performance $\eta(\pi') - \eta(\pi)$ can be decomposed as an expectation of advantages.
\begin{lemma}
\label{applem:performancedifference}
Given two policies $\pi', \pi$,
\begin{equation}
    \label{appeqn:performancedifference}
    \eta(\pi') - \eta(\pi) = \dfrac{1}{1-\gamma} \sum_{s} d_{\pi'}(s) \sum_{a}\left[\pi'(a|s)A^{\pi}(s,a)\right]
\end{equation}
\end{lemma}
\begin{proof}
Please see the proof of Lemma 6.1 in \citet{Kakade2002approximately}.
\end{proof}

The complex dependency of $d_{\pi'}(s)$ on  $\pi'$ makes \cref{appeqn:performancedifference} difficult to optimize directly. 
It is proved in \citet{achiam2017constrained} that the performance difference satisfies the following inequality.

\begin{lemma}
\label{applem:cpo}
$\forall \pi',\pi$, with $\epsilon^{\pi'} := \max _{s}\left|\E_{a \sim \pi'}\left[A^{\pi}(s, a)\right]\right|$, the following bound holds:
\begin{equation}
\begin{aligned}
\label{appeqn:cpo}
\quad \eta\left(\pi'\right)-\eta(\pi)
\geq \frac{1}{1-\gamma} \underset{\substack{s \sim d^{\pi} \\
a \sim \pi'}}{\E}\left[A^{\pi}(s, a)-\frac{2 \gamma \epsilon^{\pi'}}{1-\gamma} \mathrm{D_{TV}}(\pi' \| \pi)[s]\right] .
\end{aligned}
\end{equation}
\end{lemma}
\begin{proof}
Please see the proof of Corollary 1 in \citet{achiam2017constrained}.
\end{proof}
The bound \Eqref{appeqn:cpo} should be compared with \Eqref{appeqn:performancedifference}. The term ${\E}_{s \sim d^{\pi} a \sim \pi'}[A^{\pi}(s, a)$ in \Eqref{appeqn:cpo} is an approximation to $\eta(\pi') - \eta(\pi)$, using the state distribution $d^{\pi}$ instead of $d^{\pi'}$, which is known equal to $\eta(\pi') - \eta(\pi)$ to first order in the parameters of $\pi'$ on a neighborhood around $\pi$~\citep{Kakade2002approximately}.

\begin{lemma}[\cref{lem:kl constrained eta}]
If $\mathrm{D_{KL}}(\pi(\cdot|s)\|\beta(\cdot|s)) \leq \epsilon , \forall s$ is guaranteed, then the performance $\eta$ has the following bound
\begin{equation}
    \eta (\pi) \leq \eta(\beta)+\frac{\Vmax}{\sqrt{2}(1-\gamma)} \sqrt{\epsilon}
\end{equation}
\end{lemma}

\begin{proof}
By \cref{applem:performancedifference}, we have 
\begin{equation}
\begin{aligned}
    |\eta(\beta) - \eta(\pi)| &= \dfrac{1}{1-\gamma} \left|\sum_{s} d_{\beta}(s) \sum_{a}[(\beta(a|s)-\pi(a|s))Q^{\pi}(s,a)] \right| \\
    &\leq \dfrac{1}{1-\gamma} \sum_{s} d_{\beta}(s) \sum_{a}[\left|\beta(a|s)-\pi(a|s) \right||Q^{\pi}(s,a)]|  \\
    &\leq \dfrac{\Vmax}{1-\gamma} \sum_{s} d_{\beta}(s) \sum_{a}[\left|\beta(a|s)-\pi(a|s) \right|\\
    &= \dfrac{\Vmax}{1-\gamma} \sum_{s} d_{\beta}(s) \mathrm{D_{TV}}(\pi \| \beta)[s]\\
    &\leq \frac{\Vmax}{\sqrt{2}(1-\gamma)} \sum_{s} d_{\beta}(s) \sqrt{\mathrm{D_{KL}}(\pi \| \beta)[s]}
    \quad\quad\quad\quad\text{(Pinsker's inequality)}\\
    &\leq \frac{\Vmax}{\sqrt{2}(1-\gamma)} \sqrt{\epsilon}
\end{aligned}
\end{equation}
\end{proof}

\begin{proposition}[\cref{prop:tabular pi supp}]
\label{appprop:tabular pi supp}
For $\pi_i$ in \cref{alg:tabular}, 
$\operatorname{supp}(\pi_{i}(\cdot|s)) = \operatorname{supp}(\beta(\cdot|s)), \forall i$. When assuming the MDP has a fixed initial state distribution $d_0$, it implies $\operatorname{supp}(d^{\pi_i}(\cdot)) = \operatorname{supp}(d^{\beta}(\cdot))$ and $\operatorname{supp}(\rho^{\pi_i}(\cdot,\cdot)) = \operatorname{supp}(\rho^{\beta}(\cdot,\cdot))$
\end{proposition}
\begin{proof}
As $\frac{1}{Z(s)}\exp\left(\frac{\hat{A}^{\pi_i} (s,a)}{\lambda^*(s)}\right)>0$, it holds that $\operatorname{supp}(\pi_{i+1}(\cdot|s)) = \operatorname{supp}(\pi_i(\cdot|s)),~\forall s$. By a recursive argument, $\operatorname{supp}(\pi_{i}(\cdot|s)) = \operatorname{supp}(\beta(\cdot|s)), \forall i$. Here we use the exact definition of support~($=0$), rather than  $>$ some threshold.
As the distribution of $d^{\pi}(s)$ is induced by the transition dynamics and the policy. With $\operatorname{supp}(\pi_{i}(\cdot|s)) = \operatorname{supp}(\beta(\cdot|s))$ and the same $\Pcal$, it follows directly $\operatorname{supp}(d^{\pi_i}(\cdot)) = \operatorname{supp}(d^{\beta}(\cdot))$.
As $\rho^\pi(s,a)=d^{\pi}(s)\pi(a|s)$, it also holds that $\operatorname{supp}(\rho^{\pi_i}(\cdot,\cdot)) = \operatorname{supp}(\rho^{\beta}(\cdot,\cdot))$.
\end{proof}

\begin{proposition}[\cref{prop:tabular pe}]
\label{appprop:tabular pe}
In tabular MDP, if the offline dataset $\Dcal$ is generated by a behavior policy $\beta$ and $|\mathcal{D}| = \infty$, then we can have an exact evaluation of $Q^{\pi_{i}}$ for all $\pi_i$ in \cref{alg:tabular}.
\end{proposition}
\begin{proof}
We construct the empirical MDP $\Mcal_\Dcal$ in tabular setting following \citet{fujimoto2019off}. Specifically, $\Mcal_\Dcal$ is defined by the same action and state space as $\Mcal$, with an additional terminal state  $s_{\text{init}}$.
$\Mcal_\Dcal$  has transition probabilities  $\Pcal_{\mathcal{D}}(s'|s,a)=\frac{N(s, a, s')}{\sum_{\tilde{s}} N(s, a, \tilde{s})}$ , where  $N(s, a, s')$ is the number of times the tuple  $(s, a, s')$  is observed in  $\Dcal$ . If $\sum_{\tilde{s}} N(s, a, \tilde{s})=0$, then  $\Pcal_{\mathcal{D}}(s_{\text {init}} | s, a)=1$, where  $r(s_{\text {init}}, s, a)$ is to the initialized value of  $Q(s,a)$.

Following \cref{appprop:tabular pi supp}, any $(s,a)$ such that $\rho^{\pi_i}(s,a)>0$ must satisfy $\rho^{\beta}(s,a)>0$. With the no sampling error assumption: $|\mathcal{D}| = \infty$, those $(s,a)$ also satisfy $\Pcal_{\mathcal{D}}(s'|s,a)=\Pcal(s'|s,a)$ for all $s' \in \Scal$. Then following Lemma 1 in \citet{fujimoto2019off}, we can conclude the proof.
\end{proof}

\begin{theorem}[Strict policy improvement for each step, \cref{thm:tabular monotone}]
\label{appthm:tabular monotone}
If we have the exact tabular estimation of $Q$,
then $\pi_i$ in \cref{alg:tabular} guarantees monotonic improvement:
\begin{equation}
    Q^{\pi_{i+1}}(s,a) \geq Q^{\pi_{i}}(s,a)\quad \forall s,a.
\end{equation}
and the improvement is strict in at least one $(s,a)$ pair until the optimal support-constrained policy $\pi^*_{\Pi}$ is found:
\begin{equation}
    \pi_{i}=\pi^*_{\Pi}
\end{equation}
\end{theorem}

\begin{proof}
When assuming exact evaluation of $Q^{\pi_{i}}$ which holds if $|\Dcal|=\infty$ by \cref{appprop:tabular pe}, $\pi_{i+1}$ in \cref{alg:tabular} is the optimal solution of the following constrained optimization problem:
\begin{equation}
\label{appeqn:unify str}
\begin{aligned}
        & \pi_{i+1} = \argmax_\pi \underset{a \sim \pi}{\E} [Q^{\pi_{i}} (s,a)] \\
      & s.t. \mathrm{D_{KL}}(\pi\|\pi_{i}) \leq \epsilon \\
      & \sum_a \pi(a|s)=1, \forall s
\end{aligned}
\end{equation}
we know that $\mathrm{D_{KL}}(\pi_i\|\pi_{i}) =0 < \epsilon ~ \forall s$ is an strictly feasible solution to the optimization problem above. Therefore, ${\E}_{a \sim \pi_{i+1}} [Q^{\pi_{i}} (s,a)] \geq {\E}_{a \sim \pi_{i}} [Q^{\pi_{i}} (s,a)]~\forall s$. It implies
\begin{align*}
& Q^{\pi_{i}}(s, a) \\
=& \mathbb{E}\left[r(s_{t}, a_{t})+\gamma \sum_{s_{t+1},a_{t+1}} P(s_{t+1}|s_t,a_t)\pi_{i}(a_{t+1} | s_{t+1}) Q^{\pi_{i}}(s_{t+1}, a_{t+1}) | s_{t}=s, a_{t}=a\right] \\
\leq & \mathbb{E}\left[r(s_{t}, a_{t})+\gamma \sum_{s_{t+1},a_{t+1}} P(s_{t+1}|s_t,a_t)\pi_{i+1}(a_{t+1} | s_{t+1}) Q^{\pi_{i}}(s_{t+1}, a_{t+1}) | s_{t}=s, a_{t}=a\right] \\
&\quad\quad\quad\quad\quad \ldots \\
\leq & \mathbb{E}_{\pi_{i+1}}\left[\sum_{k=0}^{\infty} \gamma^{k} r(s_{t+k}, a_{t+k}) | s_{t}=s, a_{t}=a\right] \\
=& Q^{\pi_{i+1}}(s, a)
\end{align*}
Therefore, $Q^{\pi_{i+1}}(s,a) \geq Q^{\pi_{i}}(s,a)~ \forall s,a$. 

Then if in some iteration $i$, the optimization problem in \Eqref{appeqn:unify str} do not improve the objective function at any $s$, 
$i.e.~ {\E}_{a \sim \pi_{i+1}} [Q^{\pi_{i}} (s,a)] = {\E}_{a \sim \pi_{i}} [Q^{\pi_{i}} (s,a)]~\forall s$,
it means that $\pi_{i}$ is a minimizer of of the convex optimization problem in \Eqref{appeqn:unify str}. Note that $\pi_{i}$ is a interior point of the feasible set.
Consider another convex optimization problem whose feasible region contains the original one in \Eqref{appeqn:unify str}:
\begin{equation}
\label{appeqn:unify str inf}
\begin{aligned}
        & \pi_{i+1} = \argmax_\pi \underset{a \sim \pi}{\E} [Q^{\pi_{i}} (s,a)] \\
      & s.t. \mathrm{D_{KL}}(\pi\|\pi_{i}) < \infty\\
      & \sum_a \pi(a|s)=1, \forall s
\end{aligned}
\end{equation}
Here $\mathrm{D_{KL}}(\pi\|\pi_{i}) =\sum_a \pi(a|s)\log\frac{\pi(a|s)}{\pi_i(a|s)} < \infty$ is actually equivalent to $ \operatorname{supp}(\pi(\cdot|s))\subseteq  \operatorname{supp}(\pi_{i}(\cdot|s))=\operatorname{supp}(\beta(\cdot|s))$, and thus the feasible region of \Eqref{appeqn:unify str inf} is actually the support constrained policy set $\Pi$ (\cref{def:support contraint}). Note that the problem in  \Eqref{appeqn:unify str} and the problem in \Eqref{appeqn:unify str inf} share the same convex objective function and the convex feasible set of the former is contained by the one of the latter. Here we know $\pi_{i}$ is a interior minimizer of the former convex optimization problem. It implies that $\pi_{i}$ is also a local minimizer of the latter convex optimization problem, which further implies $\pi_{i}$ is a global minimizer of the latter convex optimization problem by convexity. Therefore,
\begin{equation}
    \label{eqn:pi_i optimal}
    \pi_{i} = \argmax_{\pi \in \Pi} \underset{a \sim \pi}{\E} [Q^{\pi_{i}} (s,a)]
\end{equation}
By Bellman Equation of $\pi_i$, we have
\begin{align}
    Q^{\pi_{i}}(s,a)&=r(s, a)+\gamma \underset{s' \sim P(\cdot|s, a)}{\E} \underset{a'\sim \pi_i(\cdot|s')}{\E} \left[Q^{\pi_{i}}(s',a')\right]\\
    &=r(s, a)+\gamma \underset{s' \sim P(\cdot|s, a)}{\E} \max_{\pi \in \Pi} \underset{a'\sim \pi(\cdot|s')}{\E} \left[Q^{\pi_{i}}(s',a')\right]
\end{align}
It is actually the support-constrained Bellman optimality equation in \cref{def:optimal support contraint policy}. So we have, $Q^{\pi_{i}}(s,a)=Q^*_{\Pi}(s,a),~\forall (s,a)$. Further, by the equivalence of \Eqref{eqn:pi_i optimal} and the definition of the optimal support-constrained policy $\pi^*_{\Pi}$ in \Eqref{eqn:support constrained optimal pi}, we know $\pi_i=\pi^*_{\Pi}$.
To conclude, if in some iteration $i$ the optimization problem in \Eqref{appeqn:unify str} do not make improvement at any $(s,a)$, it implies $Q^{\pi_{i+1}}(s,a) = Q^{\pi_{i}}(s,a)~\forall s,a$. 
Then our analysis shows $\pi_i$ equals the optimal support-constrained policy $\pi^*_{\Pi}$.

\end{proof}

From now on, we relax the assumption and  
incorporate approximation error and sampling error. Specifically, instead of tabular $Q$, we approximate the $Q$ functions by a value function class $\mathcal{F} \subseteq\left(\mathcal{S} \times \mathcal{A} \rightarrow\left[0, V_{\max }\right]\right)$.
At the same time, we do not make any assumption on $|\mathcal{D}|$.

To prove the FQE error bound of policy evaluation~(\cref{thm:fqe}), we first prove a lemma of generalization bound.

\begin{lemma}[generalization bound]
\label{applem:generalization bound}
Under the approximate completeness assumption~(\cref{asm:approx complete}), with probability at least $1-\delta$, for all $k = 1,\ldots, K$ and $\forall \pi$, we have:
\begin{equation}
    \left\|f_{k+1}-\mathcal{T}^\pi f_{k}\right\|_{2, \rho^\beta}^{2} \leq \frac{22 \Vmax^{2} \log \left(|\mathcal{F}| K / \delta\right)}{|\Dcal|}+20 \epsilon_{\text{complete}}
\end{equation}
\end{lemma}
\begin{proof}
For any fixed $f_{k-1}$, FQE deals with the following regression problem on dataset 
\begin{equation}
    f_{k} \leftarrow \argmin_{f \in \mathcal{F}} \sum_{i=1}^{|\Dcal|}\left(f(s_{i}, a_{i})-r_{i}-\gamma f_{k-1}(s_{i}', \pi(s_i'))\right)^2
\end{equation}
In this regression problem, we have $|r_{i}+\gamma f_{k-1}(s_{i}', \pi(s_i'))| \leq 1+\gamma \Vmax \leq 2\Vmax$. And for our Bayes optimal solution, we have $|\Tcal^\pi f_{k-1}(s, a)| = |r(s, a) + \gamma \E_{s'\sim P(\cdot|s,a)} f_{k-1}(s',\pi(s_i'))| \leq 1+ \gamma \Vmax \leq 2\Vmax$. Also note that \cref{asm:approx complete} implies that $\min_{f\in\Fcal} \|f - \Tcal^\pi f_{k-1}\|_{2,\rho^\beta}^2 \leq \epsilon_{\text{complete}}$. Thus, we can apply least squares generalization bound here (Lemma A.11 in \citet{agarwal2019reinforcement}).
With probability at least $1-\delta$, we have
\begin{equation*}
    \left\|{f}_{k}-\mathcal{T}^\pi f_{k-1} \right\|_{2, \rho^\beta}^{2} \leq \frac{{22 \Vmax}^{2} \log (|\mathcal{F}| / \delta)}{|\Dcal|}+20 \epsilon_{\text{complete}}
\end{equation*}
The above inequality holds for the fixed $f_{k-1}$. Since different $\pi$ can induce different $f_{k-1}$, we apply a union bound over all possible $f_{k-1} \in \Fcal$. Also, we apply a union bound over all $k=1,\ldots, K$. Therefore, with probability at least $1-\delta$, we have
\begin{equation*}
    \left\|{f}_{k}-\mathcal{T}^\pi f_{k-1} \right\|_{2, \rho^\beta}^{2} \leq \frac{{44 \Vmax}^{2} \log (|\mathcal{F}|K / \delta)}{|\Dcal|}+20 \epsilon_{\text{complete}}
\end{equation*}
\end{proof}

\begin{theorem}[FQE error bound, \cref{thm:fqe}]
\label{appthm:fqe}
Under \cref{asm:approx complete} and \cref{asm:concentrability}, with probability at least $1-\delta$, after $K$ iterations of $\rm{FQE}$, which initializes $f_0 \in \Fcal$ arbitrarily, and iterates $K$ times:
\begin{equation*}
    f_{k} \leftarrow \argmin_{f \in \mathcal{F}} \|f(s, a)-r-\gamma f_{k-1}(s', \pi(s'))\|_{2,\Dcal}
\end{equation*} 
the following bound holds:
\begin{align}
\label{appeqn:fqe}
\|&Q^\pi-f_K\|_{1,\rho^\pi} 
\leq
\frac{1-\gamma^K}{1-\gamma} \sqrt{C\epsilon_{gb}}
+\gamma^{K}\Vmax\\
& \text{where } \epsilon_{gb} := \frac{{44 \Vmax}^{2} \log (|\mathcal{F}|K / \delta)}{|\Dcal|}+20 \epsilon_{\text {complete}} 
\end{align}
\end{theorem}

\begin{proof}
we first bound $\|f_K-Q^\pi\|_{2,d^\pi_t \times \pi}$ for all time step $t$.
\begin{align}
\left\|f_{K}-Q^{\pi}\right\|_{2,d^\pi_t \times \pi} &=\left\|f_{K}-\mathcal{T}^\pi f_{K-1}+\mathcal{T}^\pi f_{K-1}-Q^{\pi}\right\|_{2,d^\pi_t \times \pi} \nonumber \\
& \leq \underbrace{\left\|f_{K}-\mathcal{T}^\pi f_{K-1}\right\|_{2,d^\pi_t \times \pi}}_{(1)}+\underbrace{\left\|\mathcal{T}^\pi f_{K-1}-\mathcal{T}^\pi Q^{\pi}\right\|_{2,d^\pi_t \times \pi}}_{(2)}
\end{align}
For term (1): following \cref{asm:concentrability} and \cref{applem:generalization bound}, with probability at least $1-\delta$, for all $k = 1,\ldots, K$ and $\forall \pi$ that satisfies \cref{asm:concentrability}, we have 
\begin{align}
\label{eq:term1}
    \left\|f_{k}-\mathcal{T}^\pi f_{k-1}\right\|_{2,d^\pi_t \times \pi} 
    &\leq \sqrt{C} \left\|f_{k}-\mathcal{T}^\pi f_{k-1}\right\|_{2,\rho^\beta} \quad\quad\quad\quad\text{(\cref{asm:concentrability})} \nonumber \\
    &\leq \sqrt{C} \sqrt{\frac{{44 \Vmax}^{2} \log (|\mathcal{F}|K / \delta)}{|\Dcal|}+20 \epsilon_{\text {complete}}} \quad\quad\quad\quad\text{(\cref{applem:generalization bound})}
\end{align}
For term (2):
\begin{align}
\label{eq:term2}
    \left\|\mathcal{T}^\pi f_{K-1}-\mathcal{T}^\pi Q^{\pi}\right\|_{2,d^\pi_t \times \pi} &=\sqrt{\mathbb{E}_{(s, a) \sim d^\pi_t \times \pi}\left[\left((\mathcal{T}^\pi f_{K-1})(s, a)-(\mathcal{T}^\pi Q^\pi)(s, a)\right)^{2}\right]} \nonumber \\
    &=\sqrt{\mathbb{E}_{(s, a) \sim d^\pi_t \times \pi}\left[\left(\gamma \mathbb{E}_{s^{\prime} \sim P(\cdot|s,a)} \E_{a' \sim \pi(\cdot|s')}\left[f_{K-1}(s',a')-Q^\pi (s',a')\right]\right)^{2}\right]} \nonumber\\
    & \leq \gamma \sqrt{\mathbb{E}_{(s, a) \sim d^\pi_t \times \pi, s^{\prime} \sim P(\cdot|s, a),a'\sim \pi(\cdot|s')}\left[\left(f_{K-1}(s',a')-Q^\pi (s',a')\right)^2\right]}
    ~\text{(Jensen's inequality)}
    \nonumber\\
    &=\gamma \sqrt{ \mathbb{E}_{(s',a') \sim d^\pi_{t+1} \times \pi}\left[\left(f_{K-1}(s',a')-Q^\pi (s',a')\right)^2\right]}\nonumber\\
    &=\gamma \left\|f_{K-1}-Q^{\pi}\right\|_{2,d^\pi_{t+1} \times \pi}
\end{align}
Combine term (1) and term (2):
\begin{equation}
    \left\|f_{K}-Q^{\pi}\right\|_{2,d^\pi_t \times \pi} 
    \leq 
    \sqrt{C\epsilon_{gb}}
    + \gamma \left\|f_{K-1}-Q^{\pi}\right\|_{2,d^\pi_{t+1} \times \pi}
\end{equation}
where 
\begin{equation}
\epsilon_{gb} := \frac{{44 \Vmax}^{2} \log (|\mathcal{F}|K / \delta)}{|\Dcal|}+20 \epsilon_{\text {complete}} 
\end{equation}
Note that we can apply the same analysis on $\left\|f_{K-1}-Q^{\pi}\right\|_{2,d^\pi_{t+1} \times \pi}$. 
We recursively repeat the same process $K$ times:
\begin{align}
\|f_K-Q^\pi\|_{2,d^\pi_t \times \pi}
&\leq 
\sum_{k=0}^{K-1} \gamma^{t}\sqrt{C\epsilon_{gb}}
+\gamma^{K}\left\|f_{0}-Q^\pi\right\|_{2,d^\pi_{t+K} \times \pi} \nonumber \\
&\leq
\frac{1-\gamma^K}{1-\gamma} \sqrt{C\epsilon_{gb}}
+\gamma^{K}\Vmax
\end{align}

Then we derive the bound $\|f_K-Q^\pi\|_{2,\rho^\pi}$ with distribution $\rho^\pi=d^\pi \times \pi$:
\begin{align*}
    \|f_K-Q^\pi\|_{2,\rho^\pi}
    &=
    \|f_K-Q^\pi\|_{2,d^\pi \times \pi}\\
    &=
    \sqrt{\sum_{s} d^\pi(s) \sum_{a} \pi(a|s) (f_K(s,a)-Q^\pi(s,a))^2}\\
    &=
    \sqrt{\sum_{s} (1-\gamma)\sum_{t=0}^{\infty} \gamma^t d^\pi_t(s) \sum_{a} \pi(a|s) (f_K(s,a)-Q^\pi(s,a))^2}\\
    &=
    \sqrt{(1-\gamma) \sum_{t=0}^{\infty} \gamma^t \sum_{s} d^\pi_t(s) \sum_{a} \pi(a|s) (f_K(s,a)-Q^\pi(s,a))^2}\\
    &=
    \sqrt{(1-\gamma) \sum_{t=0}^{\infty} \gamma^t \|f_K-Q^\pi\|_{2,d^\pi_t \times \pi}^2}\\
    &\leq
    \sqrt{(1-\gamma) \sum_{t=0}^{\infty} \gamma^t 
    \left(\frac{1-\gamma^K}{1-\gamma} \sqrt{C\epsilon_{gb}}
    +\gamma^{K}\Vmax\right)^2}\\
    &=
    \frac{1-\gamma^K}{1-\gamma} \sqrt{C\epsilon_{gb}}
    +\gamma^{K}\Vmax
\end{align*}
Finally we apply the inequality between weighted $L1$-norm and weighted $L2$-norm, and conclude the proof:
\begin{align*}
    \|f_K-Q^\pi\|_{1,\rho^\pi}
    &\leq 
    \|f_K-Q^\pi\|_{2,\rho^\pi} \quad\quad\quad\quad\text{(Jensen's inequality)}\\
    &\leq
    \frac{1-\gamma^K}{1-\gamma} \sqrt{C\epsilon_{gb}}
    +\gamma^{K}\Vmax
\end{align*}
\end{proof}

With some derivations, we can translate the $Q$ error bound to the $A$ error bound, which is closer to the performance difference~(\cref{applem:performancedifference},\cref{applem:cpo}). For simplicity of the final results, 
we assume $\alpha \leq 0.48$. A small $\alpha$ leads to a trust region update. Please note $\alpha \leq 0.48$ is not a necessary requirement here, just for simplicity of the results.
\begin{lemma}[Advantage error bound]
\label{applem:adv error}
Under \cref{asm:approx complete}, \cref{asm:concentrability} and $\alpha \leq 0.48$, with probability at least $1-\delta$, for any $\pi_{i+1},\pi_{i}$ satisfying \Eqref{eqn:str close form update apporx}, the following bound holds:
\begin{equation}
\label{appeqn:adv}
\begin{aligned}
\left|\mathbb{E}_{s \sim d_{\pi_{i}}, a \sim \pi_{i+1}}\left[A^{\pi_{i}}(s, a)-\hat{A}^{\pi_{i}}(s, a)\right]\right| \leq 2\alpha\left(\frac{1-\gamma^K}{1-\gamma}  \sqrt{C\epsilon_{gb}}
+\gamma^{K}\Vmax\right)
\end{aligned}
\end{equation} 
\end{lemma}

\begin{proof}
\begin{align}
\label{eq:advbound1}
    &\left|\mathbb{E}_{s \sim d_{\pi_{i}}, a \sim \pi_{i+1}}\left[A^{\pi_{i}}(s, a)-\hat{A}^{\pi_{i}}(s, a)\right]\right| \nonumber\\
    =& \left|\mathbb{E}_{s \sim d_{\pi_{i}}, a \sim \pi_{i+1}}\left[Q^{\pi_{i}}(s, a)-\hat{Q}^{\pi_{i}}(s, a)\right]+\mathbb{E}_{s \sim d_{\pi_{i}}, a \sim \pi_{i}}\left[\hat{Q}^{\pi_{i}}(s, a)-Q^{\pi_{i}}(s, a)\right]\right|\nonumber\\
    =& \left|\mathbb{E}_{s \sim d_{\pi_{i}}, a \sim \pi_{i}}\left[\left(\frac{\pi_{i+1}}{\pi_{i}}-1\right)\left(Q^{\pi_{i}}(s, a)-\hat{Q}^{\pi_{i}}(s, a)\right)\right]\right|\nonumber\\
    =& \left|\mathbb{E}_{s \sim d_{\pi_{i}}, a \sim \pi_{i}}\left[\left(\frac{1}{Z(s)} \exp \left(\frac{\alpha\hat{A}^{\pi_{i}}(s, a)}{\Vmax}\right)-1\right)\left(Q^{\pi_{i}}(s, a)-\hat{Q}^{\pi_{i}}(s, a)\right)\right]\right| \nonumber\\
    \leq &~ \mathbb{E}_{s \sim d_{\pi_{i}}, a \sim \pi_{i}}\left[\left|\frac{1}{Z(s)} \exp \left(\frac{\alpha\hat{A}^{\pi_{i}}(s, a)}{\Vmax}\right)-1\right|\left|Q^{\pi_{i}}(s, a)-\hat{Q}^{\pi_{i}}(s, a)\right|\right]
\end{align}
Now we bound the term $\left|\frac{1}{Z(s)} \exp \left(\frac{\alpha\hat{A}^{\pi_{i}}(s, a)}{\Vmax}\right)-1\right|$.

Because 
\begin{equation*}
    \left|\hat{A}^{\pi_{i}}(s, a)\right| = \left|\hat{Q}^{\pi_{i}}(s, a)-\mathbb{E}_{a \sim \pi_{i}}\left[\hat{Q}^{\pi_{i}}(s, a)\right]\right| \leq \left|\Vmax-0\right|=\Vmax
\end{equation*}
\begin{equation}
\label{appeqn:Zs leq}
    Z(s)=\sum_a \pi_i(a|s)\exp(\frac{\alpha \hat{A}^{\pi_i} (s,a)}{\Vmax}) 
    \leq  \sum_a \pi_i(a|s)\exp(\frac{\alpha \Vmax}{\Vmax})  = \exp(\alpha)
\end{equation}
On the other hand, by Jensen’s inequality,
\begin{equation}
\label{appeqn:Zs geq}
    Z(s)=\sum_a \pi_i(a|s)\exp(\frac{\alpha \hat{A}^{\pi_i} (s,a)}{\Vmax}) 
    \geq  \exp(\frac{\alpha \sum_a \pi_i(a|s) \hat{A}^{\pi_i} (s,a)}{\Vmax}) = 1
\end{equation}
Note that here $\sum_a \pi_i(a|s) \hat{A}^{\pi_i} (s,a)=0$. It is because we calculate $\hat{A}^{\pi_i}$ from $\hat{Q}^{\pi_i}$ directly: $\hat{A}^{\pi_i}(s,a):= \hat{Q}^{\pi_i}(s,a) - \E_{a\sim\pi_i}\hat{Q}^{\pi_i}(s,a)$.

For the numerator, $\exp(-\alpha) \leq \exp \left(\frac{\alpha\hat{A}^{\pi_{i}}(s, a)}{\Vmax}\right) \leq \exp(\alpha)$

Therefore, 
\begin{equation*}
    \exp(-2\alpha) \leq \frac{1}{Z(s)} \exp \left(\frac{\alpha\hat{A}^{\pi_{i}}(s, a)}{\Vmax}\right) \leq \exp(\alpha)
\end{equation*}
\begin{equation*}
    \left|\frac{1}{Z(s)} \exp \left(\frac{\alpha\hat{A}^{\pi_{i}}(s, a)}{\Vmax}\right)-1\right| \leq \max \left\{1-\exp(-2\alpha), \exp(\alpha)-1 \right\}
\end{equation*}

We assume $\alpha \in [0,0.48)$ for small policy update. Then it holds that for $\alpha \in [0,0.48)$, $1-\exp(-2\alpha) \geq \exp(\alpha)-1$. 

Therefore, 
\begin{equation}
\label{appeqn:abs leq}
    \left|\frac{1}{Z(s)} \exp \left(\frac{\alpha\hat{A}^{\pi_{i}}(s, a)}{\Vmax}\right)-1\right| \leq 1-\exp(-2\alpha) \leq 2\alpha
\end{equation}
So we bound the term $\left|\frac{1}{Z(s)} \exp \left(\frac{\hat{A}^{\pi_{i}}(s, a)}{{\alpha}}\right)-1\right|$ in \Eqref{eq:advbound1} with the constant $2\alpha$. Now we can conclude the proof by applying \cref{appthm:fqe} directly. 
\end{proof}

\begin{proposition}
[Trust Region, \cref{prop:trust region}]
\label{appprop:trust region}
For any $\pi_{i+1},\pi_{i}$ satisfying \Eqref{eqn:str close form update apporx}, the following policy difference bound holds
\footnote{
we assume $\alpha \leq 0.48$, which
is not a necessary requirement, just for simplicity of the final results. 
A small $\alpha$ leads to a trust region update: the closer $\alpha$ is to $0$, the closer $\pi_{i+1}$ and $\pi_{i}$ are to each other.}:
\begin{equation}
\label{appeqn:tv}
\mathrm{D_{TV}}(\pi_{i+1} \| \pi_i)[s] \leq \alpha , \forall s
\end{equation}
\begin{equation}
\label{appeqn:reverse kl}
\mathrm{D_{KL}}(\pi_{i} \| \pi_{i+1})[s] \leq \alpha , \forall s
\end{equation}
\begin{equation}
\label{appeqn:kl}
\mathrm{D_{KL}}(\pi_{i+1} \| \pi_{i})[s] \leq \frac{\alpha(e^\alpha-e^{-\alpha})}{2} , \forall s
\end{equation}
\end{proposition}

\begin{proof}

For $\mathrm{D_{TV}}(\pi_{i+1} \| \pi_{i})$:
\begin{align*}
\mathrm{D_{TV}}(\pi_{i+1} \| \pi_{i})[s]&=\frac{1}{2}\sum_a |\pi_{i+1}(a|s)-\pi_{i}(a|s)|\\
&=\frac{1}{2}\sum_a \left|\frac{\pi_{i+1}(a|s)}{\pi_{i}(a|s)}-1\right|\pi_{i}(a|s)\\
&=\frac{1}{2}\sum_a \left|\frac{1}{Z(s)}\exp \left(\frac{\alpha\hat{A}^{\pi_{i}}(s, a)}{\Vmax}\right)-1\right|\pi_{i}(a|s) \quad\quad\quad\quad\text{by \Eqref{eqn:str close form update apporx}}\\
&\leq \frac{1}{2}\sum_a 2 \alpha  \pi_{i}(a|s) \quad\quad\quad\quad\text{by \Eqref{appeqn:abs leq}}\\
&=\alpha
\end{align*}

For $\mathrm{D_{KL}}(\pi_{i} \| \pi_{i+1})$:
\begin{align*}
\mathrm{D_{KL}}(\pi_{i} \| \pi_{i+1})[s]&=\sum_a \pi_i(a|s)\log\frac{\pi_i(a|s)}{\pi_{i+1}(a|s)}\\
&=\sum_a \pi_i(a|s)\log\left[Z(s)\exp \left(\frac{-\alpha\hat{A}^{\pi_{i}}(s, a)}{\Vmax}\right)\right] \quad\quad\quad\quad\text{by \Eqref{eqn:str close form update apporx}}\\
&=\log Z(s)-\sum_a \pi_i(a|s)\frac{\alpha\hat{A}^{\pi_{i}}(s, a)}{\Vmax}\\
&=\log Z(s)\leq \alpha \quad\quad\quad\quad\text{by \Eqref{appeqn:Zs leq}}
\end{align*}

For $\mathrm{D_{KL}}(\pi_{i+1} \| \pi_{i})$:
\begin{align}
\mathrm{D_{KL}}(\pi_{i+1} \| \pi_{i})[s]&=\sum_a \pi_{i+1}(a|s)\log\frac{\pi_{i+1}(a|s)}{\pi_{i}(a|s)} \nonumber\\
&=\sum_a \pi_{i+1}(a|s)\log\left[\frac{1}{Z(s)}\exp \left(\frac{\alpha\hat{A}^{\pi_{i}}(s, a)}{\Vmax}\right)\right] \quad\quad\quad\quad\text{by \Eqref{eqn:str close form update apporx}} \nonumber\\
&=\sum_a \pi_{i+1}(a|s)\frac{\alpha\hat{A}^{\pi_{i}}(s, a)}{\Vmax} - \log Z(s) \nonumber\\
&\leq \sum_a \frac{1}{Z(s)}\pi_i(a|s)\exp(\frac{\alpha \hat{A}^{\pi_i} (s,a)}{\Vmax})\frac{\alpha\hat{A}^{\pi_{i}}(s, a)}{\Vmax}-0 \quad\quad\quad\quad\text{by \Eqref{eqn:str close form update apporx} and \Eqref{appeqn:Zs geq}} \nonumber\\
&\leq \sum_a \pi_i(a|s)\exp(\frac{\alpha \hat{A}^{\pi_i} (s,a)}{\Vmax})\frac{\alpha\hat{A}^{\pi_{i}}(s, a)}{\Vmax} \quad\quad\quad\quad\text{by \Eqref{appeqn:Zs geq}}
\label{appeqn:KLupper}
\end{align}
An obvious upper bound of \Eqref{appeqn:KLupper} is $\alpha\exp(\alpha)$ obtained by substituting $\hat{A}^{\pi_{i}}(s, a)$ with $\Vmax$.

To derive a tighter upper bound of \Eqref{appeqn:KLupper}, consider the following problem, where $N=|\Acal|$:
\begin{align*}
&\max \sum_{i=1}^N y_i x_i\exp(x_i)\\
s.t. &\sum_{i=1}^N y_i x_i=0,~\sum_{i=1}^N y_i=1,~ -\alpha \leq x_i \leq \alpha
\end{align*}
For $\alpha \in [0,0.48)$ and for any fixed $y$, the optimization problem is convex with respect to $x$.
Therefore, in the optimal solution, $x_i$ can be found at the boundary. So we assume $x_1 \sim x_k = -\alpha$, $x_{k+1} \sim x_N = \alpha$. The optimization problem becomes:
\begin{align*}
&\max \sum_{i=1}^k y_i (-\alpha)\exp(-\alpha)+\sum_{i=k+1}^N y_i \alpha\exp(\alpha)\\
&s.t. \sum_{i=1}^k y_i (-\alpha)+\sum_{i=k+1}^N y_i\alpha=0,~\sum_{i=1}^N y_i=1
\end{align*}
The constraint implies $\sum_{i=1}^k y_i=\sum_{i=k+1}^N y_i=1/2$. Then the objective is equal to $\frac{\alpha(e^\alpha-e^{-\alpha})}{2}$, which concludes the proof.
\end{proof}

\begin{theorem}[Safe policy improvement for each step, \cref{thm:safe policy improvement}]
\label{appthm:safe policy improvement}
Under \cref{asm:approx complete} and \cref{asm:concentrability}, for $\pi_{i+1},\pi_{i}$ satisfying \Eqref{eqn:str close form update apporx}, with $\epsilon^{\pi_{i+1}}:=\max _{s}|\mathbb{E}_{a \sim \pi_{i+1}}[A^{\pi_i}(s,a)]|$, the  following performance difference bound holds:
\begin{equation*}
  \eta(\pi_{i+1})-\eta(\pi_i) \geq \frac{\Vmax}{(1-\gamma)\alpha} \E_{s\sim d^{\pi_i}}[\mathrm{D_{KL}}(\pi_{i+1} \| \pi_i)] -\frac{2\gamma\epsilon^{\pi_{i+1}}}{(1-\gamma)^2}\alpha - \frac{2\alpha}{1-\gamma} \left(\frac{1-\gamma^K}{1-\gamma}  \sqrt{C\epsilon_{gb}}
+\gamma^{K}\Vmax\right)
\end{equation*}
\end{theorem}

\begin{proof}
We start with \cref{applem:cpo}:
\begin{align*}
\eta\left(\pi_{i+1}\right)-\eta(\pi_i)
&\geq \frac{1}{1-\gamma} \underset{\substack{s \sim d^{\pi_i} \\
a \sim \pi_{i+1}}}{\E}\left[A^{\pi_i}(s, a)-\frac{2 \gamma \epsilon^{\pi_{i+1}}}{1-\gamma} \mathrm{D_{TV}}(\pi_{i+1} \| \pi_i)[s]\right]\\
&\geq \frac{1}{1-\gamma} \underset{\substack{s \sim d^{\pi_i} \\
a \sim \pi_{i+1}}}{\E}A^{\pi_i}(s, a)-\frac{2\gamma\epsilon^{\pi_{i+1}}}{(1-\gamma)^2}\alpha
\quad\quad\quad\quad\text{by the TV bound in \cref{appprop:trust region}}\\
&\geq \frac{1}{1-\gamma} \underset{\substack{s \sim d^{\pi_i} \\
a \sim \pi_{i+1}}}{\E}\hat{A}^{\pi_i}(s, a)-\frac{2\alpha}{1-\gamma} \left(\frac{1-\gamma^K}{1-\gamma}  \sqrt{C\epsilon_{gb}}
+\gamma^{K}\Vmax\right)-\frac{2\gamma\epsilon^{\pi_{i+1}}}{(1-\gamma)^2}\alpha
\quad\quad\text{by \cref{applem:adv error}}\\
\end{align*}
Please note that $\pi_{i+1}$ is the optimal solution of the optimization problem in \Eqref{eqn:str optimization}:
\begin{equation*}
\pi_{i+1} = \argmax_\pi ~ \alpha \underset{\substack{s \sim d^{\pi_i} \\
a \sim \pi}}{\E} [\hat{A}^{\pi_i} (s,a)] -\Vmax \underset{s \sim d^{\pi_i}}{\E} [\mathrm{D_{KL}}(\pi\|\pi_i)]
\end{equation*}
and $\pi_i$ is a feasible solution that makes the objective $0$.
Therefore, 
\begin{equation*}
\alpha \underset{\substack{s \sim d^{\pi_i} \\
a \sim \pi_{i+1}}}{\E} [\hat{A}^{\pi_i} (s,a)] -\Vmax \underset{s \sim d^{\pi_i}}{\E} [\mathrm{D_{KL}}(\pi_{i+1}\|\pi_i)] \geq 0
\end{equation*}
Now we can conclude the proof:
\begin{equation*}
  \eta(\pi_{i+1})-\eta(\pi_i) \geq \frac{\Vmax}{(1-\gamma)\alpha} \E_{s\sim d^{\pi_i}}[\mathrm{D_{KL}}(\pi_{i+1} \| \pi_i)[s]] -\frac{2\gamma\epsilon^{\pi_{i+1}}}{(1-\gamma)^2}\alpha - \frac{2\alpha}{1-\gamma} \left(\frac{1-\gamma^K}{1-\gamma}  \sqrt{C\epsilon_{gb}}
+\gamma^{K}\Vmax\right)
\end{equation*}

\end{proof}

\section{Derivations of the EAWBC Framework}
\label{app:EAWBC framework}

At the $i^{th}$ iteration, the unified algorithm solves the following constrained optimization problem to update the policy
\begin{equation}
\begin{aligned}
\label{appeqn:unify}
        & \pi_{i+1} = \argmax_\pi \underset{a \sim \pi}{\E} [\hat{A}^{\pi_{\text{pe}}} (s,a)] \\
       s.t.~ \mathrm{D_{KL}}&(\pi\|\pi_{\text{base}})[s] \leq \epsilon,\quad
      \sum_a \pi(a|s)=1,~ \forall s
\end{aligned}   
\end{equation}
The constrained optimization problem in \Eqref{appeqn:unify} is convex, and the Lagrangian is:
\begin{equation}
\begin{aligned}
\label{appeqn:Lagrangian}
\mathcal{L}(\pi, \lambda, \nu)= & \underset{a \sim \pi}{\E} \left[\hat{A}^{\pi_{\text{pe}}} (s,a)\right]
+\lambda[\epsilon-\mathrm{D_{KL}}(\pi(\cdot|s))\|\pi_{\text{base}}(\cdot|s))]+\nu\left(\sum_a \pi(a|s)-1\right)
\end{aligned}   
\end{equation}
The KKT condition gives:
\begin{equation}
\begin{aligned}
\label{appeqn:KKT}
\frac{\partial \mathcal{L}}{\partial \pi}=\hat{A}^{\pi_{\text{pe}}} (s,a)+\lambda \log \pi_{\text{base}}(a|s)-\lambda \log \pi(a|s)-\lambda+\nu=0
\end{aligned}
\end{equation}
Solving for $\pi$ gives the closed form solution $\pi^*$:
\begin{equation}
\begin{aligned}
\label{appeqn:unify close form 1}
     \pi^*(a|s) = &\pi_{\text{base}}(a|s)\exp(\frac{\hat{A}^{\pi_{\text{pe}}} (s,a)+\nu-\lambda}{\lambda})
\end{aligned}
\end{equation}
By the condition $\sum_a \pi^*(a|s)=1$, we have
\begin{align}
     &\sum_a \pi_{\text{base}}(a|s)\exp(\frac{\hat{A}^{\pi_{\text{pe}}} (s,a)+\nu-\lambda}{\lambda})=1\\ 
     \Rightarrow~  &\exp(\frac{\lambda-\nu}{\lambda})=\sum_a \pi_{\text{base}}(a|s)\exp(\frac{\hat{A}^{\pi_{\text{pe}}} (s,a)}{\lambda})\\
     \Rightarrow~  &\nu=\lambda-\lambda\log\left[\sum_a \pi_{\text{base}}(a|s)\exp(\frac{\hat{A}^{\pi_{\text{pe}}} (s,a)}{\lambda}) \right]
\label{appeqn:nu}
\end{align}
Now consider the dual problem to solve for the Lagrangian multiplier $\lambda^*$.
By Substitute \Eqref{appeqn:KKT} into \Eqref{appeqn:Lagrangian}, we have
\begin{equation}
\max_\pi \mathcal{L}=\epsilon\lambda+\lambda - \nu
\end{equation}
Then Substituting $\nu$ with \Eqref{appeqn:nu}, we obtain the dual function:
\begin{equation}
g(\lambda)=\max_\pi \mathcal{L}=\epsilon \lambda +\lambda\log\left[\sum_a \pi_{\text{base}}(a|s)\exp(\frac{\hat{A}^{\pi_{\text{pe}}} (s,a)}{\lambda}) \right]
\end{equation}
Therefore, we can obtain $\lambda^*(s)$ by solving the following convex dual problem:
\begin{equation}
     \lambda^*(s)=\argmin_{\lambda \geq 0}~ \epsilon \lambda + \lambda \log\left[\sum_a \pi_{\text{base}}(a|s) \exp\left(\frac{\hat{A}^{\pi_{\text{pe}}} (s,a)}{\lambda}\right)\right]
\end{equation}
Now we replace the term $\exp(\frac{\nu-\lambda}{\lambda})$ in \Eqref{appeqn:unify close form 1} with a per-state normalizing factor $Z(s)$ and finally present the analytical solution of the constrained optimization problem in \Eqref{appeqn:unify}:
\begin{equation}
\begin{aligned}
\label{appeqn:unify close form}
     \pi_{i+1}(a|s) = &\pi_{\text{base}}(a|s)f(s,a;\pi_{\text{pe}})\\
     \text{where } f(s,a;\pi_{\text{pe}})&:=\frac{1}{Z(s)}\exp\left(\frac{\hat{A}^{\pi_{\text{pe}}} (s,a)}{\lambda^*(s)}\right)\\
     Z(s):=\sum_a &\pi_{\text{base}}(a|s)\exp\left(\frac{ \hat{A}^{\pi_{\text{pe}}} (s,a)}{\lambda^*(s)}\right)
\end{aligned}
\end{equation}

In practice, the non-parametric solution in \Eqref{appeqn:unify close form} can be projected onto the parametric policy class by minimizing the KL divergence:
\begin{align}
&\argmin_{\phi} \underset{s \sim \Dcal}{\E} \left[\mathrm{D_{KL}}(\pi_{i+1}(\cdot|s)\|\pi_\phi(\cdot|s))\right]\\
=&\argmin_{\phi} \underset{s \sim \Dcal,a \sim \pi_{i+1}(\cdot|s)}{\E} \left[ \log\left(\frac{\pi_{i+1}(a|s)}{\pi_\phi(a|s)}\right)\right]\\
=&\argmin_{\phi} \underset{s \sim \Dcal,a \sim \pi_{base}(\cdot|s)}{\E} \left[ f(s,a;\pi_{\text{pe}}) \log\left(\frac{\pi_{i+1}(a|s)}{\pi_\phi(a|s)}\right)\right] \quad\text{by \Eqref{appeqn:unify close form}}\\
=&\argmax_{\phi} \underset{s \sim \Dcal,a \sim \pi_{base}(\cdot|s)}{\E} \left[ f(s,a;\pi_{\text{pe}}) \log\left(\pi_\phi(a|s)\right)\right]
\label{appeqn:unify KL}
\end{align}

\section{Experimental Details and Extended Results}
\label{app:implementation}
\subsection{Reasonableness for STR  to omit $Z(s)$ in practice}
\label{app:omit normalize}
In \Eqref{eqn:pi objective}, as all prior EAWBC works, we omit the normalization factor $Z(s)$,
because it only affects the relative weight of different states in the training objective, not different actions. The EAWBC objective in \Eqref{eqn:eawbc obj is} is derived from the minimization of $\mathrm{D_{KL}}(\pi_{i+1}(\cdot|s)\|\pi_\phi(\cdot|s))$. 
Since the weight at each state do not have specific meaning, 
we can minimize this KL divergence under any distribution  whose support is equal to $d^{\pi_{i+1}}(s)$. And because of the equal-support property of STR $\operatorname{supp}(d^{\pi_{i+1}}(\cdot)) = \operatorname{supp}(d^{\beta}(\cdot))$ by \cref{prop:tabular pi supp}, $d^{\beta}(s)$ is a qualified distribution, which allows us to sample directly from $\Dcal$ to optimize. 
Since the density of $d^{\beta}$ already differs from $d^{\pi_{i+1}}$, it has no meaning to restore the correct $d^{\beta}(s)$ density by the normalization factor $Z(s)$, which is empirically hard to estimate and will introduce more instability.

\subsection{Experimental Details}
\label{app:implementation_details}

\begin{table}[htbp]
\caption{Hyperparameters of policy training in STR.}

\label{tab:hyper_str}
\begin{center}
\begin{small}
\begin{tabular}{cll}
\toprule
                              & Hyperparameter          & \multicolumn{1}{l}{Value}           \\ \midrule
\multirow{12}{*}{STR}         & Optimizer               & \multicolumn{1}{l}{Adam ~\citep{kingma2014adam}}            \\
                              & Critic learning rate    & \multicolumn{1}{l}{$3\times 10^{-4}$}            \\
                              & Actor learning rate     & \multicolumn{1}{l}{$3\times 10^{-4}$ with cosine schedule}  \\
                              & Batch size              & 256                                 \\
                              & Discount factor                & 0.99                                \\
                              & Number of iterations    & $10^6$                             \\
                              & Target update rate $\tau$      & 0.005                               \\
                              & Policy update frequency & 2                                   \\
                              & Number of Critics & 4                                   \\
                              & Temperature $\lambda$   & \multicolumn{1}{l}{\{0.5, 2\} for Gym-MuJoCo} \\
                              &                         & \multicolumn{1}{l}{\{0.1\} for AntMaze}    \\
                              &Variance of Gaussian Policy  & 0.1              \\
                              \midrule
\multirow{2}{*}{Architecture} & Actor    & input-256-256-output                                 \\
                              & Critic & input-256-256-1                                      \\ \bottomrule
\end{tabular}
\end{small}
\end{center}
\end{table}

For the MuJoCo locomotion tasks, we average returns over 10 evaluation trajectories and 5 random seeds, while for the Ant Maze tasks, we average over 100 evaluation trajectories and 5 random seeds. Following the suggestions of the authors of the dataset, we subtract 1 from the rewards for the Ant Maze datasets. We choose TD3~\citep{fujimoto2018addressing} as our base algorithm and optimize a deterministic policy. To compute the importance sampling ratio, we need the density of any action under the deterministic policy. For this, we assume all policies are Gaussian with a fixed variance $0.1$. Note that the only hyperparameter we tuned is the temperature $\lambda$. We use $\lambda=0.1$ for Ant Maze tasks and $\lambda = \{0.5,2\}$ for MuJoCo locomotion tasks ($\lambda=2$ for expert and medium-expert datasets, $\lambda=0.5$ for medium, medium-replay, random datasets). And following previous work ~\citep{brandfonbrener2021offline}, 
we clip exponentiated advantages to $(-\infty, 100]$.
All hyperparameters are included in \Cref{tab:hyper_str}.

\subsection{From Theoretical to Practical}
The practical STR algorithm takes larger update steps~(smaller $\lambda$) than what theory recommends, which is common among trust region methods. In addition, the behavior density is estimated using a specific model, which will inevitably have errors. However, compared with other algorithms that require the behavior density, STR is less susceptible to such errors. This is because STR only needs to query the behavior density of the in-dataset $(s,a)$ pairs, i.e., $\hat \beta(a|s)$ where $(s,a)\sim \Dcal$, and therefore does not require much generalization ability of the model, making it relatively easier to estimate accurately.

\subsection{KL Divergence between Trained Policy and Behavior Policy.}
\label{app:KL}

\begin{figure}[ht]
	\centering
	\includegraphics[width=0.9\linewidth]{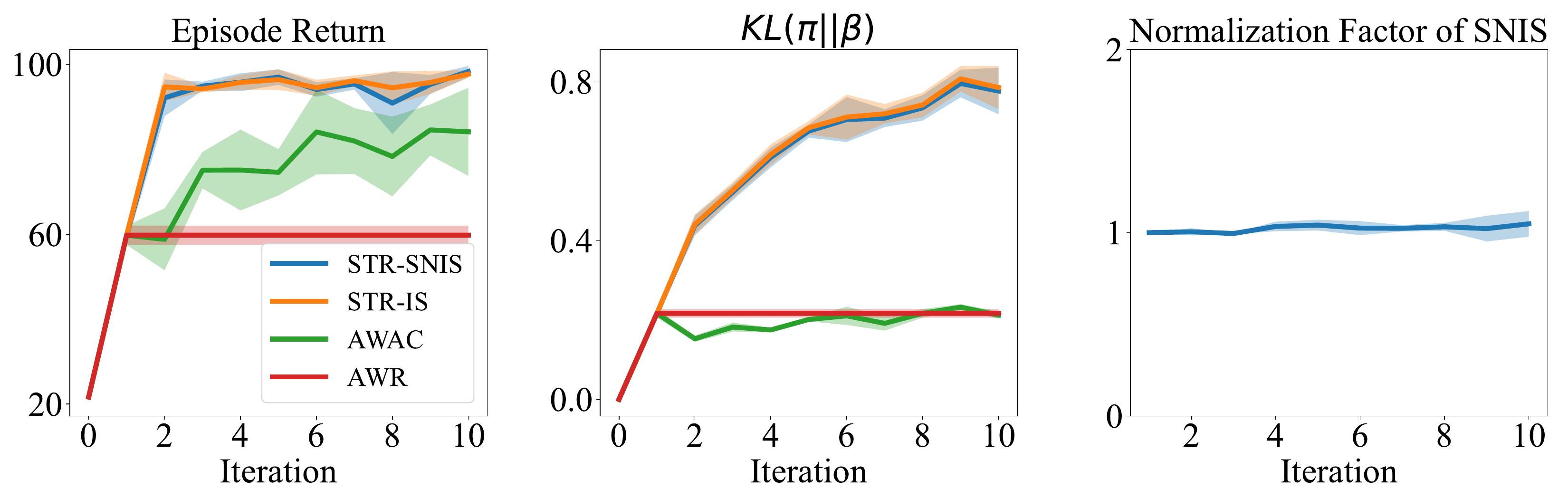}
	\caption{Extended Results on hopper-medium-replay-v2. The experimental setup is the same as that in \cref{sec:Experimental Verification of the Theories}.
	}
	\label{fig:app_convergence}
\end{figure}

In this section, we investigate the KL divergence between the trained policy and the behavior policy under various EAWBC algorithms. We conduct experiments in the same setup as \cref{sec:Experimental Verification of the Theories}~(\cref{fig:convergence}) for hopper-medium-replay-v2. Both policy evaluation and policy improvement are trained to convergence at each iteration, and all algorithms adopt the same hyperparameter that controls the constraint strength. The results are shown in \cref{fig:app_convergence}.
Compared to STR-SNIS and STR-IS, AWAC and AWR have substantially smaller $\mathrm{KL}(\pi||\beta)$ and inferior performance due to their implicit density constraint.
On the other hand, $\mathrm{KL}(\pi||\beta)$ of STR-IS and STR-SNIS increases significantly with iteration. With a less restrictive support constraint, STR is able to deviate more from the behavior policy~(larger $\mathrm{KL}(\pi||\beta)$) to achieve better performance.
Note that STR-IS and STR-SNIS yield similar results. This is because the normalization factor of SNIS is actually very close to $1$.
Therefore, IS plays a critical role in STR by relaxing the density constraint to support constraint while SNIS is a non-critical component.
In addition, compared with IS, the variance reduction effect of SNIS does not implicitly decreases $\mathrm{KL}(\pi||\beta)$~\citep{chatterjee2018sample}.

\subsection{Experimental Verification of Theories in the Tabular Setting}
\label{app:Experimental Verification of Theories in the Tabular Setting}
\begin{figure}[ht]
	\centering
	\begin{subfigure}{0.23\textwidth}
    \includegraphics[width=1\textwidth]{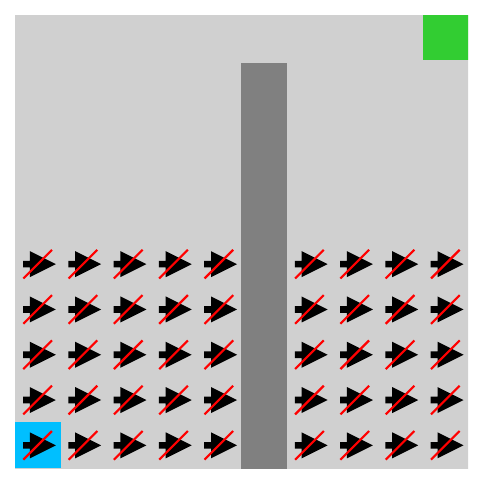}
    \vspace{1mm}
    \caption{Toy Maze}
    \label{fig:app_tabular_maze}
    \end{subfigure}
    \hspace{5mm}
    \begin{subfigure}{0.55\textwidth}
    \includegraphics[width=\textwidth]{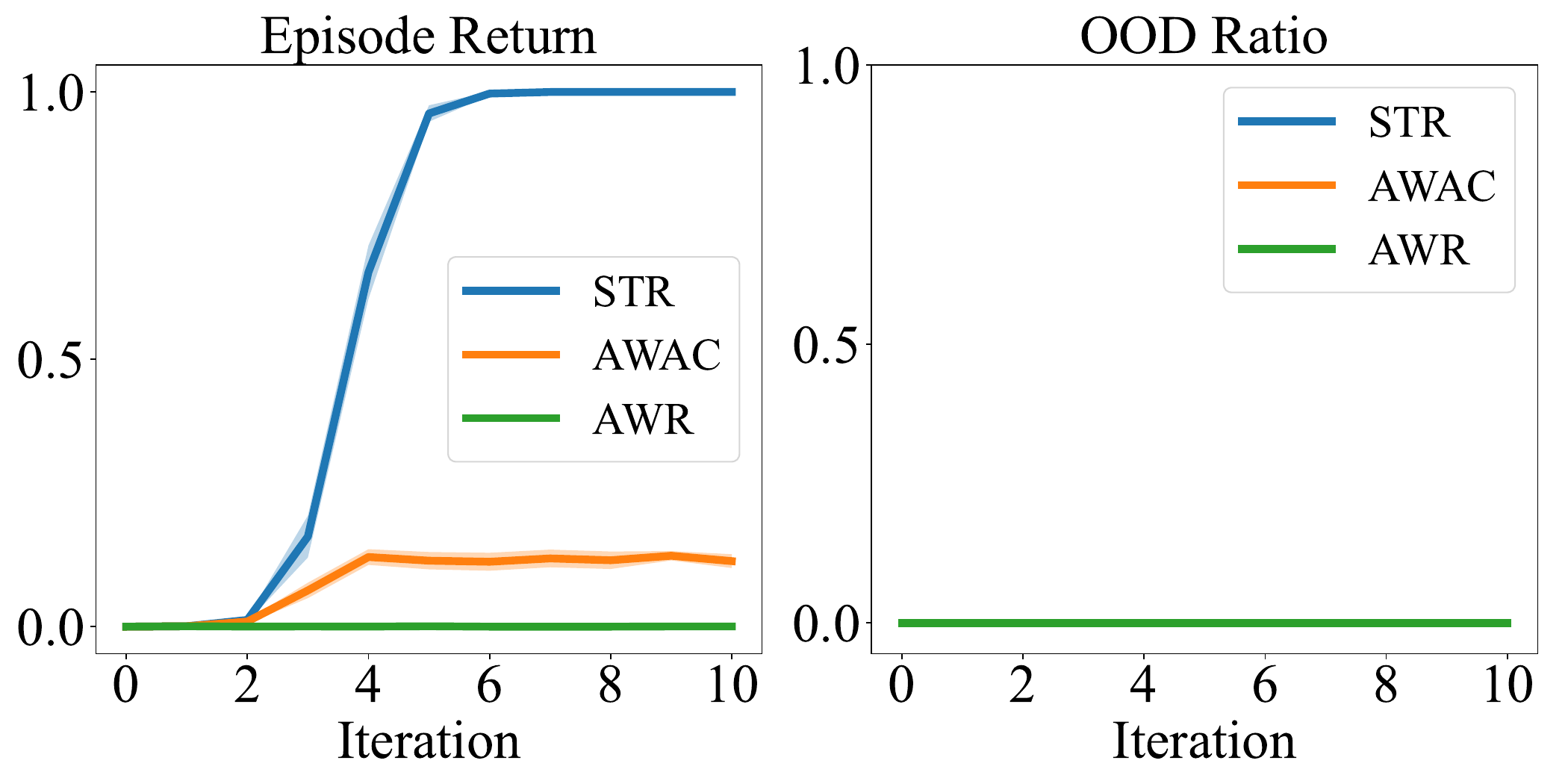}
    \caption{Learning Curves on the Toy Maze}
    \label{fig:app_tabular_curve}
    \end{subfigure}
	\caption{(a) The maze environment with OOD actions. (b) The Episode Return of STR, AWAC and AWR~(averaged over 1000 trajectories). (c) The OOD Ratio of STR, AWAC and AWR. The curves are averaged over 5 seeds, with the shaded area representing the standard deviation across seeds. SVR guarantees strict policy improvement until convergence to the optimal support-constrained policy.
	}
	\label{fig:app_tabular}
\end{figure}
To verify the theories in \cref{sec:str tabular}, we conduct a $10\times10$ maze experiment. As depicted in \cref{fig:app_tabular_maze}, the task is to navigate from  bottom-left to top-right, with a wall in the middle. The agent receives a reward of $1$ upon reaching the goal. Episodes are terminated after $25$ steps and $\gamma$ is set to $0.9$. We first collect $10,000$ transitions using a random policy. Then we remove all the transitions containing rightward actions in the lower half of the maze, so that the rightward action in that region is Out-of-Distribution (OOD). For some policy $\pi$, the OOD Ratio indicates the proportion of $(s,a)$ not in the dataset among all $(s,a)$ pairs satisfying $\pi(a|s)>0$.
The results are shown in \cref{fig:app_tabular_curve}. The performance of STR increases monotonically with iteration until it converges to the optimal policy, verifying \cref{thm:tabular monotone}. Furthermore, the policy obtained by STR is completely within the behavior support~($0$ OOD Ratio). For AWAC and AWR, due to being implicitly subject to the more restrictive $\mathrm{KL}(\pi||\beta)$ density constraint, their policies are also within the support. However, they do not have strictly increasing performance and have no guarantees at convergence. This is consistent with our theoretical analysis~(\cref{tab:awbc theory}).

\subsection{Learning Curves of STR on MuJoCo and Antmaze Tasks}
\label{app:learning curves}
Learning curves of STR on MuJoCo locomotion tasks and Antmaze tasks are presented in \cref{fig:str_mujoco_appendix} and \cref{fig:str_ant_appendix} respectively. The curves are averaged over 5 seeds, with the shaded area representing the standard deviation across seeds.
\begin{figure}[ht]
	\centering
	\includegraphics[width=\linewidth]{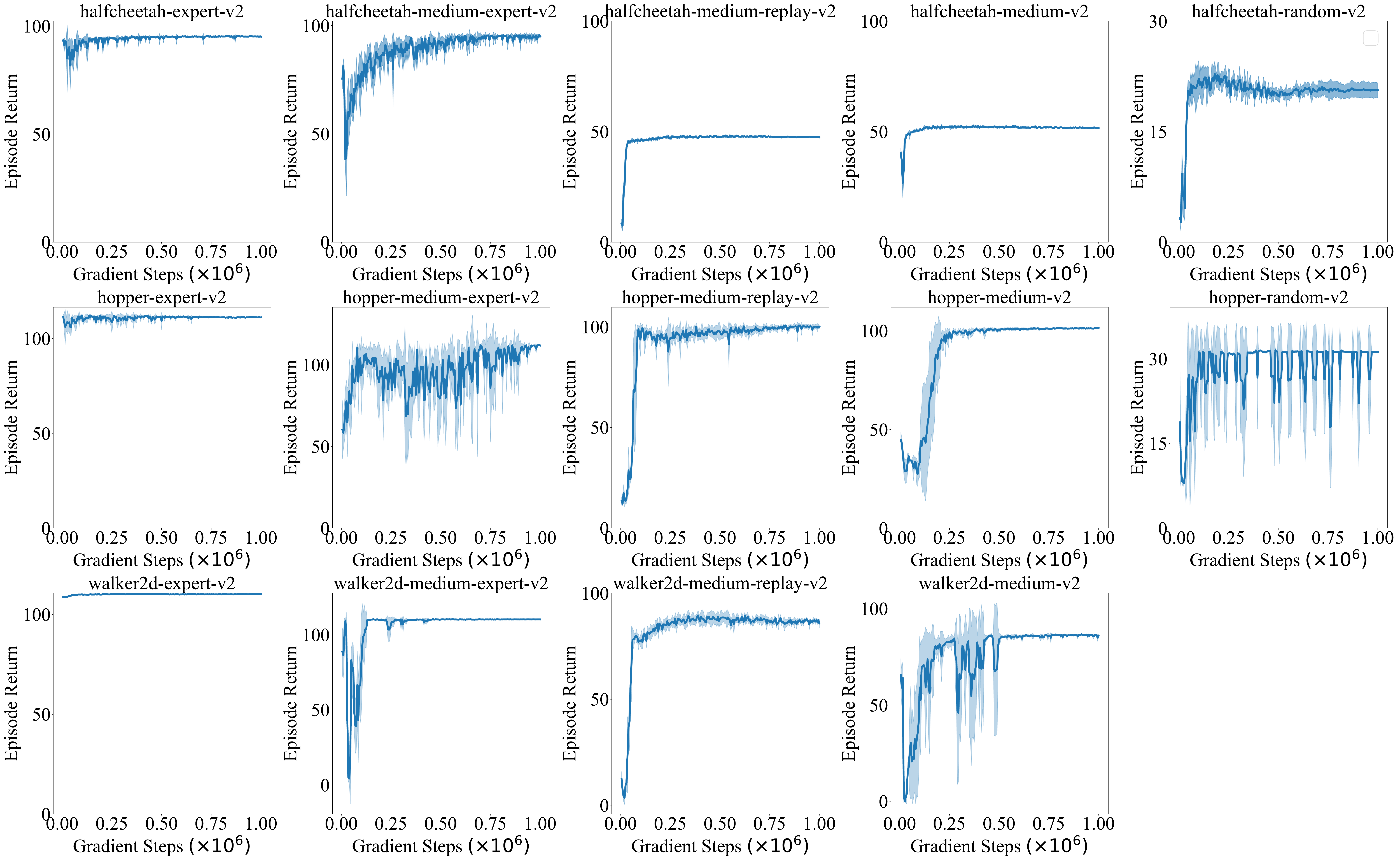}
	\caption{Learning Curves of STR on MuJoCo Locomotion Tasks.
	} 
	\label{fig:str_mujoco_appendix}
\end{figure}
\begin{figure}[ht]
	\centering
	\includegraphics[width=0.85\linewidth]{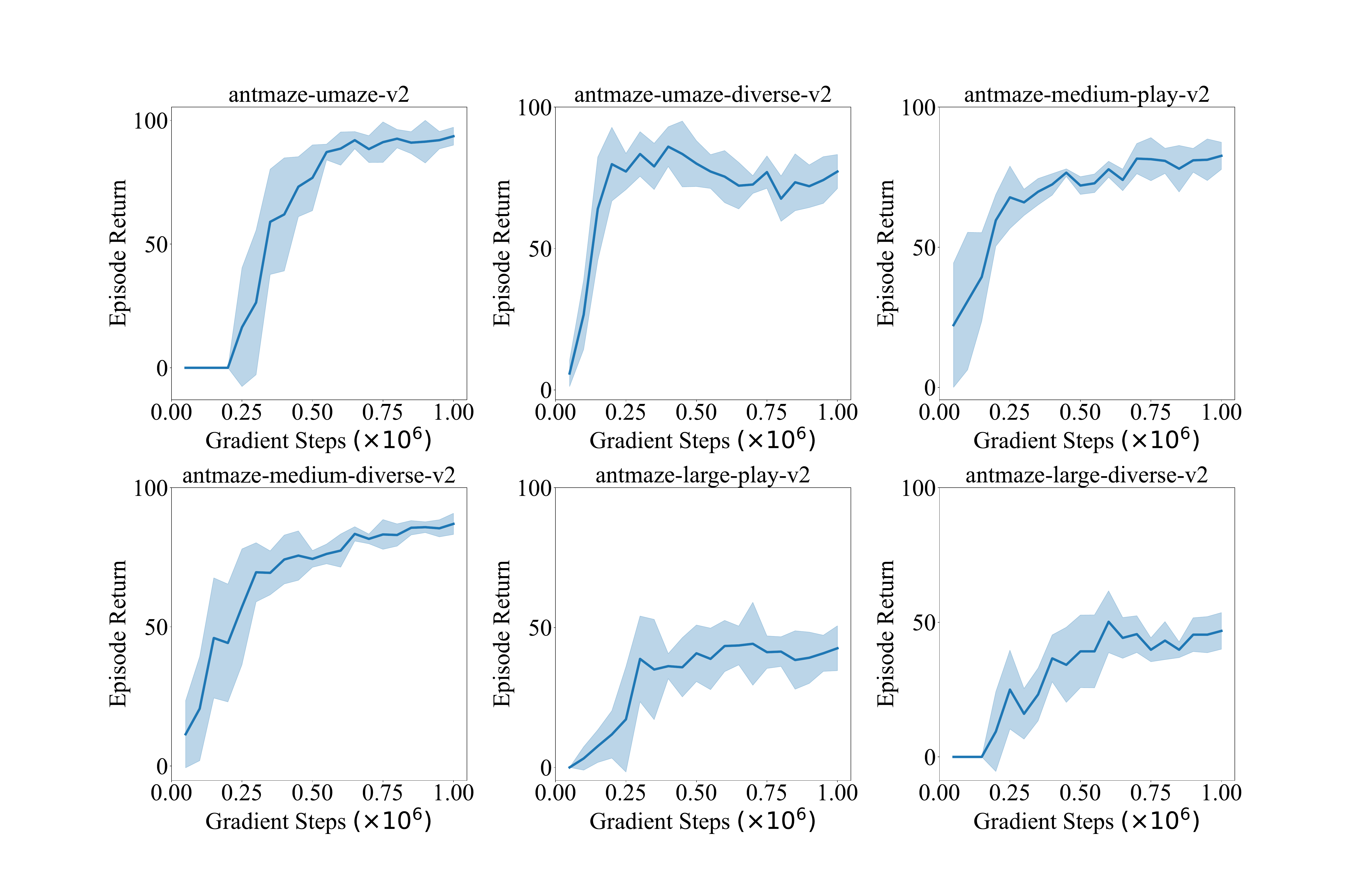}
	\vspace{-5mm}
	\caption{Learning Curves of STR on AntMaze Tasks.
	}
	\label{fig:str_ant_appendix}
\end{figure}

\end{document}